\documentclass[11pt,a4paper]{article}

\usepackage[table]{xcolor}  %

\usepackage[acceptedWithA]{tacl2021v1}

\usepackage{times,latexsym}
\usepackage{url}
\usepackage[T1]{fontenc}
\usepackage{xspace,mfirstuc,tabulary}

\newif\iftaclinstructions
\taclinstructionsfalse %
\iftaclinstructions
\renewcommand{\confidential}{}
\renewcommand{\anonsubtext}{(No author info supplied here, for consistency with
TACL-submission anonymization requirements)}
\newcommand{\instr}
\fi

\iftaclpubformat %

\else

\fi

\usepackage[T1]{fontenc}    %
\usepackage{hyperref}       %
\usepackage{url}            %
\usepackage{booktabs}       %
\usepackage{amsfonts}       %
\usepackage{nicefrac}       %
\usepackage{microtype}      %
\usepackage{lipsum}     %
\usepackage{graphicx}
\usepackage{mathtools}
\usepackage{graphbox}
\usepackage{doi}
\usepackage{amsmath,amssymb}
\usepackage{thmtools}
\usepackage{thm-restate}

\usepackage{caption}
\usepackage{amssymb}
\usepackage{amsmath}

\usepackage{bm} %
\usepackage{cleveref}
\hypersetup{
    colorlinks,
    linkcolor={red!50!black},
    citecolor={blue!50!black},
    urlcolor={blue!80!black}
}
\crefname{example}{example}{examples}
\Crefname{example}{Example}{Examples}
\DeclareCaptionType{example}[Example][List of examples]
\usepackage{marginnote}

\usepackage{amsthm}

\usepackage{fancyvrb}

\usepackage{algorithm}
\usepackage[noEnd]{algpseudocodex}
\usepackage{etoolbox}

\usepackage{subcaption}
\expandafter\def\csname ver@subfig.sty\endcsname{}

\usepackage{caption}

\usepackage[most]{tcolorbox}
\usepackage{todonotes}

\definecolor{TodoColor}{rgb}{1,0.7,0.6}

\newtheorem{theorem}{Theorem}[section]

\newtheorem{lemma}[theorem]{Lemma}
\newtheorem{remark}[theorem]{Remark}
\newtheorem{proposition}[theorem]{Proposition}

\makeatletter
\newcommand\IfRestateTF{%
  \ifx\label\thmt@gobble@label %
    \expandafter\@firstoftwo
  \else
    \expandafter\@secondoftwo
  \fi
}
\makeatother

\DeclareMathOperator*{\argmax}{arg\,max}

\makeatletter
\renewenvironment{proof}[1][\proofname]{\par
  \vspace{-\topsep}%
  \pushQED{\qed}%
  \normalfont
  \topsep0pt \partopsep0pt %
  \trivlist
  \item[\hskip\labelsep
        \itshape
    #1\@addpunct{.}]\ignorespaces
}{%
  \popQED\endtrivlist\@endpefalse
  \addvspace{12pt plus 6pt} %
}
\makeatother

\newcounter{mydefinition}
\newcommand{\mydefinition}[1]{
    \refstepcounter{mydefinition}
    \paragraph{Definition \thesection.\themydefinition: #1}
}

\DeclareUnicodeCharacter{0301}{circlemarker}

\title{Tokenization as Finite-State Transduction}

\author{
  Marco Cognetta
  \\
  Tokyo Institute of Technology \\
  \texttt{cognetta.marco@gmail.com}
  \And
  Naoaki Okazaki
  \\
  Tokyo Institute of Technology
  \\
  \texttt{okazaki@c.titech.ac.jp}
}

\date{}

\begin{document}
\maketitle
\begin{abstract}
Tokenization is the first step in modern neural language model pipelines where an input text is converted to a sequence of \textit{subword} tokens.
We introduce from first principles a finite-state transduction framework which can efficiently encode all possible tokenizations of a regular language.
We then constructively show that Byte-Pair Encoding (BPE) and MaxMatch (WordPiece), two popular tokenization schemes, fit within this framework.
For BPE, this is particularly surprising given its resemblance to context-free grammar and the fact that it does not tokenize strings from left to right.

An application of this is to \textit{guided generation}, where the outputs of a language model are constrained to match some pattern. Here, patterns are encoded at the character level, which creates a mismatch between the constraints and the model's subword vocabulary. While past work has focused only on constraining outputs without regard to the underlying tokenization algorithm, our framework allows for simultaneously constraining the model outputs to match a specified pattern while also adhering to the underlying tokenizer's canonical tokenization.
\end{abstract}

\section{Introduction}

Modern NLP models typically operate on \textit{subword} tokens, an intermediate granularity between characters and words, which allows for efficiently encoding any input sequence with a finite vocabulary. Recently, using finite-state transducers to produce a regular language of \textit{subword} sequences that match a given regular language of \textit{character} sequences (i.e., the concatenation of the subword sequence would match the character-level pattern) has found an application to \textit{guided generation} for LLMs. However, these subword-level patterns are typically unaware of any underlying tokenization scheme. Here, we consider the problem of generating subword-level regular languages that simultaneously match a character-level regular language, but also only admit subword sequences that would have been produced by a tokenization algorithm such as MaxMatch or Byte-Pair Encoding (BPE).

Our contributions are a clear explanation of how one would construct a subword-level automaton from first principles via automata theory. We first introduce a character-to-subword transducer, which allows us to promote character-level patterns to subword-level patterns (\Cref{sec:tokenization_agnostic_pattern_promotion}). 
However, these subword-level patterns do not account for an underlying tokenization scheme. We extend the framework to two popular subword tokenization algorithms so that we can simultaneously match a character-level pattern while only only allowing subword sequences which would match the output of the underlying tokenizer. First, we consider MaxMatch-style tokenizers, which is a straightforward extension in the transducer framework (\Cref{ssec:maxmatch_transducer}), and then BPE tokenizers, which, at first glance, are not obviously representable by finite-state machines (\Cref{ssec:bpe_transducer}). For BPE specifically, we show that, contrary to the naïve runtime analysis, the BPE promotion can be done in polynomial time in the size of the input pattern and subword vocabulary (Theorem \ref{thm:bpe_complexity}). 
The application of subword-level guided generation and our framework is discussed in Section \ref{sec:guided_generation}, including a theoretical shortcoming which motivates our tokenization-preserving subword-promotion construction.

\section{Tokenization}

Tokenization algorithms convert sequences of characters into sequences of tokens drawn from a subword vocabulary (e.g., \texttt{t\textvisiblespace o\textvisiblespace k\textvisiblespace e\textvisiblespace n\textvisiblespace s} $\rightarrow$ \texttt{token\textvisiblespace s}).

\subsection{Subword Vocabulary}
Given a finite set of atomic characters $\Sigma$, a subword vocabulary, $\Gamma$, is a finite set $\Sigma \subseteq \Gamma \subset \Sigma^+$. The first set inclusion property ensures that the subword is \textit{open-vocabulary} --- any sequence built from atomic characters (i.e., a sequence $w \in \Sigma^+$) can be represented by the subword vocabulary.

\subsection{MaxMatch Encoding}
MaxMatch (or WordPiece) encoding is a tokenization algorithm that, given a subword vocabulary $\Gamma$, tokenizes an input sequence \textit{greedily} from left to right.
It iteratively selects the longest possible matching token from $\Gamma$ and adds that to the tokenized output.
While a canonical training procedure exists \cite{DBLP:conf/icassp/SchusterN12}, any subword vocabulary can be used with MaxMatch.

\begin{algorithm}[H]
\begin{algorithmic}[1]
    
    \State $i \gets 0,~t \gets \langle~\rangle$
    \State $m \gets \max_{v \in \Gamma} |v|$
    
    \While{$i < |w|$}
        \State $z \gets w_i$
        \For{$j \in 1...m$}
            \If{$w_{i:i+j} \in \Gamma$}
                \State $z \gets w_{i:i+j}$
            \EndIf
        \EndFor
        \State $\textsc{Append}(t, z)$
        \State $i \gets i + |z|$
    \EndWhile
    \State \Return $t$
\end{algorithmic}

\caption{%
    \hspace{-1mm}: MaxMatch Tokenization \newline
    \textbf{Input}: Vocabulary $\Gamma$, String $w \in \Sigma^+$\\
    \textbf{Output}: Tokenized sequence $t \in \Gamma^+$
}
\label{alg:maxmatch_inference}
\end{algorithm}

Algorithm \ref{alg:maxmatch_inference} describes MaxMatch tokenization inference. The naïve implementation of Algorithm \ref{alg:maxmatch_inference} takes $O(|w|m)$ time, where $m = \max_{v \in \Gamma} |v|$ is the length of the longest token in $\Gamma$.
However, \newcite{song-etal-2021-fast} show that MaxMatch tokenization can be viewed as a variant of the Aho-Corasick algorithm \cite{AhoC75} and describe an $O(|w|)$-time tokenization algorithm. Figure \ref{fig:maxmatch_example} provides an example of MaxMatch tokenization.

\begin{figure}
\begin{center}
\resizebox{0.7\linewidth}{!}{
\begin{tabular}{c|ccccccc}
\toprule
Step & & & & & & & \\ \midrule
0. & $\texttt{b}$ & $\texttt{a}$ & $\texttt{n}$ & $\texttt{a}$ & $\texttt{n}$ & $\texttt{a}$ & $\texttt{s}$ \\
1. & \cellcolor{blue!25}$\texttt{b}$ & \cellcolor{blue!25}\texttt{a} & \cellcolor{blue!25}\texttt{n} & \cellcolor{blue!25}\texttt{a} & \texttt{n} & \texttt{a} & \texttt{s} \\
2. & \cellcolor{blue!25}$\texttt{b}$ & \cellcolor{blue!25}\texttt{a} & \cellcolor{blue!25}\texttt{n} & \cellcolor{blue!25}\texttt{a} & \cellcolor{red!25}\texttt{n} & \cellcolor{red!25}\texttt{a} & \texttt{s} \\
3. &\cellcolor{blue!25}$\texttt{b}$ & \cellcolor{blue!25}\texttt{a} & \cellcolor{blue!25}\texttt{n} & \cellcolor{blue!25}\texttt{a} & \cellcolor{red!25}\texttt{n} & \cellcolor{red!25}\texttt{a} & \cellcolor{green!25}\texttt{s} \\
\bottomrule
\end{tabular}
}
\end{center}
\caption{MaxMatch encoding of \texttt{bananas} with \\ $\Gamma = \{\texttt{a, b, n, s, ba, na, ban, bana}\}$. At each step, the longest matching token is added to the tokenized sequence to produce \texttt{bana\textvisiblespace na\textvisiblespace s}.}\label{fig:maxmatch_example}
\end{figure}

\subsection{Byte-Pair Encoding}

Byte-Pair Encoding (BPE) is a two-stage algorithm for producing subword tokenizations \cite{Gage1994ANA, sennrich-etal-2016-neural, zouhar-etal-2023-formal}. 
In this paper, we assume that we have already-trained BPE tokenizers and never use the training procedure directly. However, we include it here to show that the ordering of merges is not arbitrary. First, a subword vocabulary is formed, given a corpus $C$ and a desired vocabulary size $k$. Assume $\textsc{Count}(C, a, b)$ counts the number of occurrences of the sequence $a b$ in corpus $C$, Algorithm \ref{alg:bpe_training} produces the vocabulary and merge list for the BPE tokenizer by iteratively finding the pair $(a, b)$ with the highest cooccurence count in the corpus, merging them, adding the merge to an ordered list of merges, and adding the token $ab$ to the subword vocabulary. This repeats until the target vocabulary size is reached.

\begin{algorithm}[H]
\begin{algorithmic}[1]

    \State $\Gamma = \Sigma,~\mu = \langle~\rangle$
    \For{$i \in 1\dots k$}
        \State $(a, b) = \argmax_{(a, b) \in C} \textsc{Count}(C, a, b)$
        \State $\textsc{Append}(\mu, (a, b))$
        \State $\textsc{Append}(\Gamma, ab)$
        \State $\textsc{Merge}(C, (a, b))$
    \EndFor
    \State \Return $(\Gamma, \mu)$

\end{algorithmic}

\caption{%
    \hspace{-1mm}: BPE Training \newline
    \textbf{Input}: Corpus $C$ over alphabet $\Sigma$, Target merge size $k$\\
    \textbf{Output}: Vocabulary $\Gamma$, Merge list $\mu$
}
\label{alg:bpe_training}
\end{algorithm}
\begin{algorithm}[H]
\begin{algorithmic}[1]
    \State $t \gets w$
    \State $\psi \gets \langle (t_i, t_{i+1})~\mid~(t_{i}, t_{i+1}) \in \mu \rangle$
    \While{$\psi \ne \emptyset$}
        \State $(a, b) \gets \Call{Max}{\psi}$
        \State $t \gets \textsc{Apply}(t, (a, b))$
        \State $\psi \gets \langle (t_i, t_{i+1})~\mid~(t_{i}, t_{i+1}) \in \mu \rangle$
    \EndWhile
    \State \Return $t$
\end{algorithmic}

\caption{%
    \hspace{-1mm}: BPE Tokenization \newline
    \textbf{Input}: BPE Tokenizer $\mathcal{B} = (\Gamma, \mu)$, String $w \in \Sigma^+$\\
    \textbf{Output}: Tokenized sequence $t \in \Gamma^+$
}
\label{alg:bpe_inference}
\end{algorithm}

A trained BPE tokenizer $\mathcal{B} = (\Gamma, \mu)$ can then be used to produce a tokenized sequence from an input character sequence, according to the merge rules. Given a sequence, BPE iteratively finds the highest priority --- the lowest index in the merge list --- merge that is present in the current (partially) tokenized sequence and merges it. This is done repeatedly until no more merges are available, at which point the tokenized sequence is returned, as shown in Algorithm \ref{alg:bpe_inference}.

Figure \ref{fig:bpe_example} gives an example of each merge being iteratively applied to an input string. Notice that successive merges do not necessarily happen in left-to-right order, and that the merging procedure resembles a tree.

\begin{figure}[t]
\centering
\includegraphics[width=0.75\linewidth]{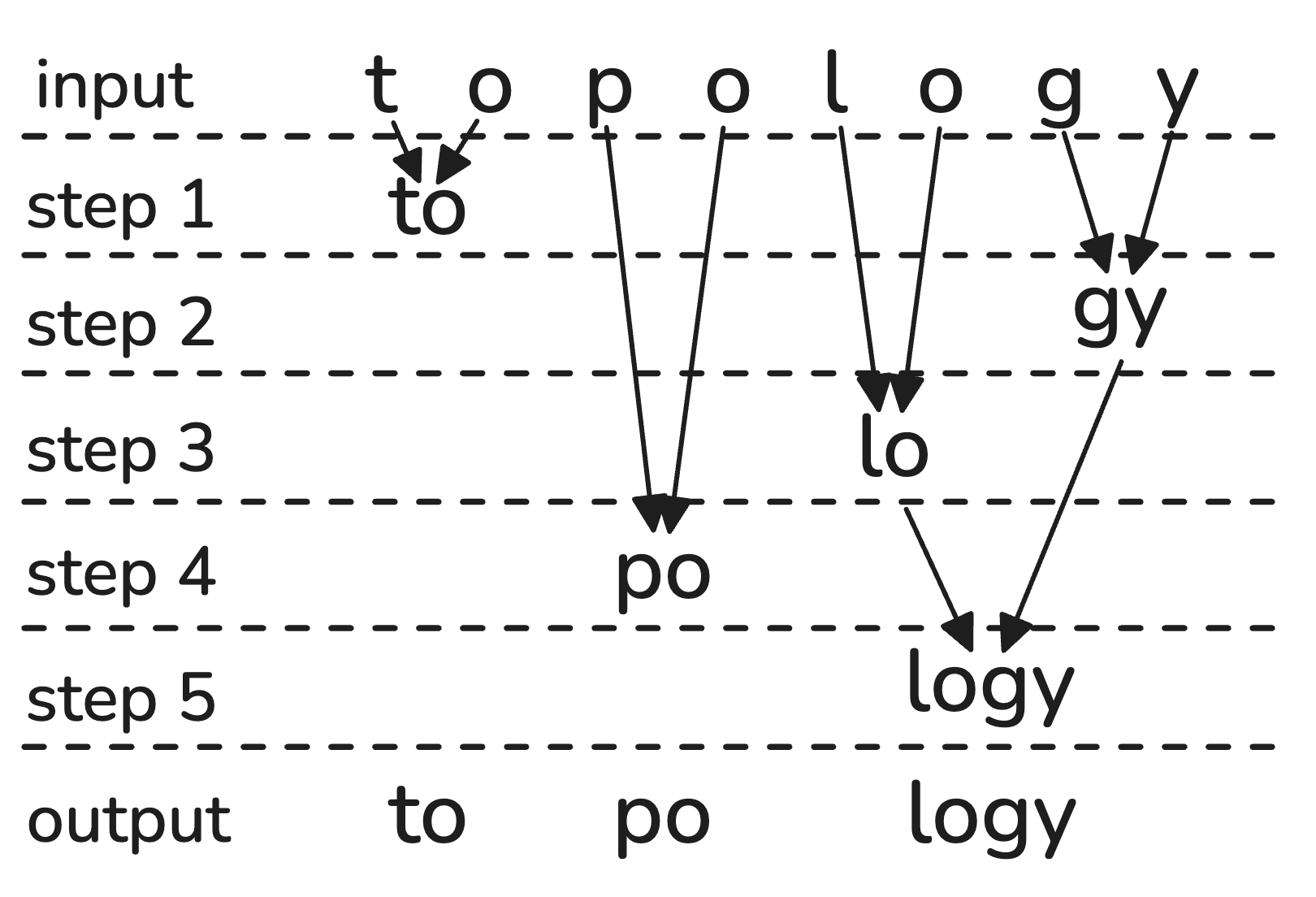}
\caption{An example of BPE tokenization given $\mu = \langle${(\texttt{t},\texttt{~o}), (\texttt{g},\texttt{~y}), (\texttt{l},\texttt{~o}), (\texttt{p},\texttt{~o}), (\texttt{lo},\texttt{~gy}) }$\rangle$. Notice that the merges are not necessarily done left to right or in order of length. The final tokenized sequence is \texttt{to\textvisiblespace po\textvisiblespace logy}.}
\label{fig:bpe_example}
\end{figure}

\section{Automata and Transducers}\label{sec:automata_and_transducers}

We introduce the fundamental concepts from automata theory that will be used in the main results of this paper. In general, we provide canonical pseudocode only for the definitions which will draw on it later in proofs, and provide only the definitions for the remaining concepts.

\mydefinition{Finite-state Transducers} are a generalization of finite automata, which instead of accepting or rejecting strings, accept or reject \textit{pairs} of strings, $(x, y)$, where $x$ is the \textit{input} string and $y$ is the \textit{output} string. In simpler terms, it accepts or rejects transformations of strings, where the input is transformed into the output.

A transducer is defined as $\mathcal{T} = (\Sigma, \Gamma, Q, q_\mathrm{start}, F, \delta)$ where $\Sigma$ and $\Gamma$ are the finite \textit{input} and \textit{output} vocabularies, $Q$ is a set of states, $q_\mathrm{start} \in Q$ is the \textit{initial} state, $F \subseteq Q$ is a set of \textit{final} states, and $\delta$ is a set of \textit{transitions} function $Q \times \Sigma \times \Gamma \times Q$ (an initial state, an input character, an output character, and a target state).
$\Sigma$ and $\Gamma$ each also implicitly contain an \textit{empty character}, $\varepsilon$.

A \textit{path} is a sequence of transitions: \[\scriptstyle \pi~=~(q_0, i_0, o_0, q_1),~(q_1, i_1, o_1, q_2),~\dots,~(q_{n-1}, i_{n-1}, o_{n-1}, q_n)\] where $(q_i, i_{i+1}, c_{i+1}, q_{i+1}) \in Q \times \Sigma \times \Gamma \times Q$. Let $s_i$ and $s_o$ be functions that concatenate the input and output words (the elements from $\Sigma$ and $\Gamma$) of a path, respectively. Then, a path is denoted $\pi_{x,y}$ if $s_i(\pi) = x$ and $s_o(\pi) = y$. A path $\pi_{x,y}$ is called \textit{accepting} if for all transitions in the path sequence, $(q_i, i_i, o_i, q_{i+1}) \in \delta$, $q_0 = q_\mathrm{start}$, and $q_n \in F$, and is called rejecting otherwise. Note that $\varepsilon$ can appear in a transition (as either the input or output character or both) but does not ``consume'' a character from the input or output strings since $\varepsilon$ is an empty character (e.g.,  $ab\varepsilon c = abc$).

We write the functional form of a transducer as $\mathcal{T}(x, y)$ for $x \in \Sigma^*$ and $y \in \Gamma^*$.
This function returns true if there exists an accepting path $\pi_{x,y}$ and false otherwise.

\mydefinition{Finite-state Automata} are a special case of transducers where $\Gamma = \Sigma$ and $\forall (q, i, o, p) \in \delta$,~$i = o$.
Consequently, automata are transducers which map strings to themselves. As such, we omit the output label from the transition definition where clear; i.e., $(q, i, p)$ rather than $(q, i, i, p)$.
An automaton is \textit{deterministic} if there are no $\varepsilon$-transitions and, for all states and for each character in the automaton's alphabet, there is at most one outgoing transition from that state labeled with that character.
Like transducers, we write $\mathcal{A}(x)$ for $x \in \Sigma^*$ as the functional form of automata, returning true if there is a valid accepting path $\pi_{x}$ and false otherwise.

For convenience, we use a functional-notation for $\delta$ when clear for deterministic automata: $\delta(q, c) \coloneq p$ if $(q, c, p) \in \delta$.
We denote the \textit{language} of a transducer, the set of (pairs of) strings accepted by it, as $L(\mathcal{T}) = \{(x, y) \mid (x, y) \in \Sigma^* \times \Gamma^* \wedge \mathcal{T}(x, y)\}$ or $L(\mathcal{T})$, respectively. Likewise, for automata, we write $L(\mathcal{A}) = \{x \mid x \in \Sigma^* \wedge \mathcal{A}(x)\}$.

\paragraph{Notational Aside}
We make use of several notational shortcuts when convenient. When referring to states, we write $q \in Q$ or $q \in \mathcal{A}$ or $\mathcal{T}$ interchangeably (for automata and transducers, respectively). Likewise, $|\mathcal{A}|$ and $|\mathcal{T}|$ are defined as the number of states in the automaton or transducer.

For transitions, we often treat $\delta$ as a (partial) function:
\begin{itemize}
    \item $\delta(q, i, o, p)$ is a transition, used interchangeably with $\delta(q, i, o) = p$ since all transducers in this paper are deterministic
    \item $\delta(q)$, is the set of transitions $\{(q, i, o, p) \mid (q, i, o, p) \in \delta\}$
\end{itemize}

\mydefinition{Label Projection.}
Suppose we have a transducer $\mathcal{T} = (\Sigma, \Gamma, Q, q_\mathrm{start}, F, \delta)$. The unary output projection function $\textsc{Proj}$ transforms the transducer into a new transducer $\mathcal{A} = (\Gamma, \Gamma, Q, q_\mathrm{start}, F, \delta'$), where $\Gamma$, $Q$, $q_\mathrm{start}$, and $F$ remain the same, but $\delta' = \{(q, o, p) \mid (q, i, o, p) \in \delta\}$. This has the effect of transforming a transducer into an automaton that computes: 
\begin{align}
\mathcal{A}(y) = \bigvee_{x \in \Sigma^*} \mathcal{T}(x, y).
\end{align}
That is, it accepts the set of all strings $y \in \Gamma^*$ which have a valid transduction from some string $x \in \Sigma^*$.

\mydefinition{Transducer Composition.}
A fundamental binary operation on transducers is composition.
Let $\mathcal{T}_1 = (\Sigma, \Xi, Q_1, q_\mathrm{start}, F_1, \delta_1)$ and $\mathcal{T}_2 = (\Xi, \Gamma, Q_2, q'_\mathrm{start}, F_2, \delta_2)$ be transducers with a shared alphabet $\Xi$ on the output side of $\mathcal{T}_1$ and the input of $\mathcal{T}_2$.
Then, composition $\circ$ produces a new transducer such that: 
\begin{align}
& \mathcal{T}_1 \circ \mathcal{T}_2(x, z) = \nonumber \\ 
& \bigvee_{y \in \Xi^*}  \forall (x, z) \in \Sigma^* \times \Gamma^*: \mathcal{T}_1(x, y) \wedge \mathcal{T}_2(y, z) \hspace{8mm}\null \nonumber  \\[-10mm]
\null
\end{align}
\vspace{0mm}

\noindent
This transducer accepts pairs of strings $(x, z) \in \Sigma^* \times \Gamma^*$ if there is an intermediate string $y \in \Xi^*$ such that $\mathcal{T}_1$ accepts $(x, y)$ and $\mathcal{T}_2$ accepts $(y, z)$.

The composition algorithm proceeds like automata intersection (indeed, intersection is just a special case of composition, where both inputs are automata). For each state $q_i, q'_j \in Q_1 \times Q_2$, the composed transducer has a state $q_{ij}$. For each $(q_i, c_i, c_o, p_j) \in \delta_1$ and $(q'_k, c'_i, c'_o, p'_l) \in \delta_2$ such that $c_o = c'_i$, there is a transition $(q_{ij}, c_i, c'_o, p_{jl})$. The initial state is  $(q_\mathrm{start}, q'_\mathrm{start})$ and the final states are all $q_{ij}$ such that $(q_i, q'_j) \in F_1 \times F_2$.

\Cref{alg:generic_composition} gives the pseudocode for generic, unweighted transducer composition\footnote{In general, generic composition is restricted to $\varepsilon$-free transducers, but it can work for specific transducers (over idempotent semirings and without $\varepsilon$-cycles), which is true for the class of transducers we consider.} \cite{DBLP:journals/ijfcs/AllauzenM09}. %

\begin{algorithm}
{\fontsize{9}{9}\selectfont
\begin{algorithmic}[1]
\State $Q_c \gets \{(q_\mathrm{start}, q'_\mathrm{start})\}$
\State $F_c \gets \{\}$
\State $\delta_c \gets \text{empty transition function}$
\State $\mathrm{queue} \gets \{(q_\mathrm{start}, q'_\mathrm{start})\}$
\While{$\mathrm{queue}$ is not empty}
    \State $(q, q') \gets \mathrm{queue}.\textsc{pop}()$
    \If{$(q, q') \in F_1 \times F_2$}
        \State $F_c \gets F_c \cup \{(q, q')\}$
    \EndIf
    \For{$(q, i, \xi_1, p) \in \delta_1$ and $(q', \xi_2, o, p') \in \delta_2(y)$}
        \If{$\xi_1 = \xi_2$ or $\xi_1 = \varepsilon$ or $\xi_2 = \varepsilon$} \label{line:transition_match}
        \State $Q_c \gets Q_c \cup \{(p, p')\}$
        \State $\mathrm{queue}.\textsc{enqueue}((p, p'))$
        \State $\delta_c \gets \delta_c \cup \{((q, q'), i, o, (p,p'))\}$
        \EndIf
    \EndFor
\EndWhile
\State \Return $\mathcal{T}_c = (\Sigma, \Gamma, Q_c, (q_\mathrm{start}, q'_\mathrm{start}), F_c, \delta_c)$

\end{algorithmic}
}
\caption{%
    \hspace{-1mm}: Generic Transducer Composition \newline
    \textbf{Input}: $\mathcal{T}_1 = (\Sigma, \Xi, Q_1, q_\mathrm{start}, F_1, \delta_1)$,\\
    \null\hspace{10.5mm} $\mathcal{T}_2 = (\Xi, \Gamma, Q_2, q'_\mathrm{start}, F_2, \delta_2)$ \newline
    \textbf{Output}: Composition $\mathcal{T}_1 \circ \mathcal{T}_2$
}
\label{alg:generic_composition}
\end{algorithm}

\mydefinition{$\bm{\varepsilon}$-Closure and Removal.}
Like automata, transducers are closed under the Kleene star operation, sometimes known as $\varepsilon$-closure \cite{DBLP:books/daglib/0086373}. Given a transducer  $\mathcal{T} = (\Sigma, \Gamma, Q, q_\mathrm{start}, F, \delta)$, define $L(\mathcal{T})^0 = \{(\varepsilon, \varepsilon)\}$,  $L(\mathcal{T})^1 = L(\mathcal{T})$, and $L(\mathcal{T})^2  = \{(x_1x_2, y_1y_2) \in \Sigma^* \times \Gamma^*$ such that $(x_1, y_1)$ and $(x_2, y_2) \in L(\mathcal{T})$. For $k \ge 1$, define $L(\mathcal{T})^k = \{(x_1x_2\dots x_k, y_1y_2\dots y_k) \mid \forall i \in [1, k], (x_i, y_i) \in L(\mathcal{T})\}$.
Then, the $\varepsilon$-closure of $\mathcal{T}$ is defined as $\mathcal{T}^* = \bigcup_{k=0}^{\infty} L(\mathcal{T})^k$, and is realized by adding a transition $(q, \varepsilon, \varepsilon, q_\mathrm{start})$ for each $q \in F$. 

$\varepsilon$-Removal for automata produces a new automaton with no $\varepsilon$-transitions that recognizes the same language as the original. It can be done in polynomial time to the size of the input automaton, but may introduce non-determinism \cite{DBLP:books/daglib/0086373}.

\section{Tokenization-Agnostic Pattern Promotion} \label{sec:tokenization_agnostic_pattern_promotion}

\subsection{Subword Lexicon Transducer}\label{sec:subword_transducer}

\setcounter{mydefinition}{0}

We now introduce the key ingredient to our framework: the character-to-subword lexicon transducer. This is a transducer that maps character sequences to subword units, e.g., \texttt{t\textvisiblespace o\textvisiblespace k\textvisiblespace e\textvisiblespace n} $\rightarrow$ \texttt{tok\textvisiblespace en}. The useful aspect of this transducer is that it can produce a succinct representation of \textit{all} possible subword tokenizations of a given input, according to a subword vocabulary, $\Gamma$.
That is, if $\{\texttt{tok}, \texttt{to}, \texttt{en}, \texttt{k}\} \subseteq \Gamma$, then it will encode the mappings to both \texttt{tok\textvisiblespace en} and \texttt{to\textvisiblespace k\textvisiblespace en}.

\paragraph{Building a lexicon transducer.}

A lexicon transducer recognizes the relation:

\begin{equation}\label{eq:lexicon_formulation}
\mathcal{L}( \mathcal{T}) = \bigcup\limits_{w \in \Gamma} (\texttt{w}_1\texttt{\textvisiblespace}\texttt{w}_2\texttt{\textvisiblespace} \dots\texttt{\textvisiblespace}\texttt{w}_n, w)
\end{equation}

A simple way to build a lexicon transducer is to construct a trie transducer which maps characters to subwords and then optionally minimize it. The final result is an acyclic, deterministic transducer for which each valid path encodes a unique word in the lexicon, and $\varepsilon$-closure allows it to transduce arbitrarily long sequences \cite{DBLP:conf/fsmnlp/CognettaA019}.\footnote{Algorithm \ref{alg:lexicon_construction} in Appendix \ref{apx:tokenization_agnostic_lexicon} gives a pseudocode implementation which is used in the proofs found there.}
\begin{figure*}[t]
\centering
\subcaptionbox{An input pattern automaton $\mathcal{A}$. \label{sfig:input_automaton}}
  {%
    \includegraphics[width=0.49\linewidth,align=c]{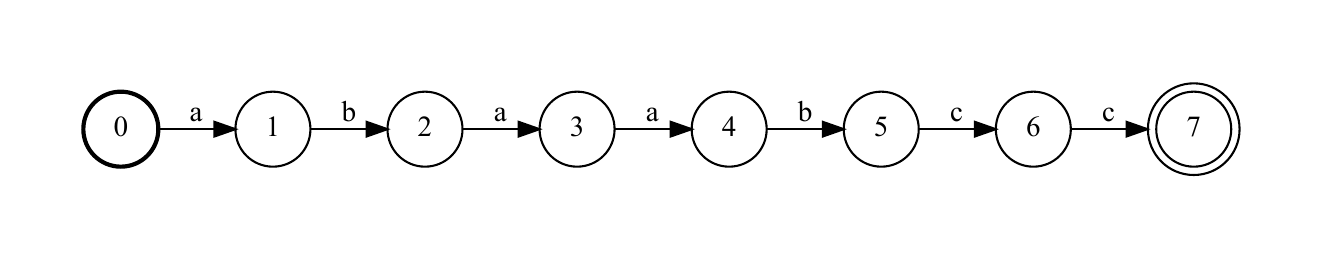}%
    \vphantom{\includegraphics[width=0.4\linewidth,align=c]{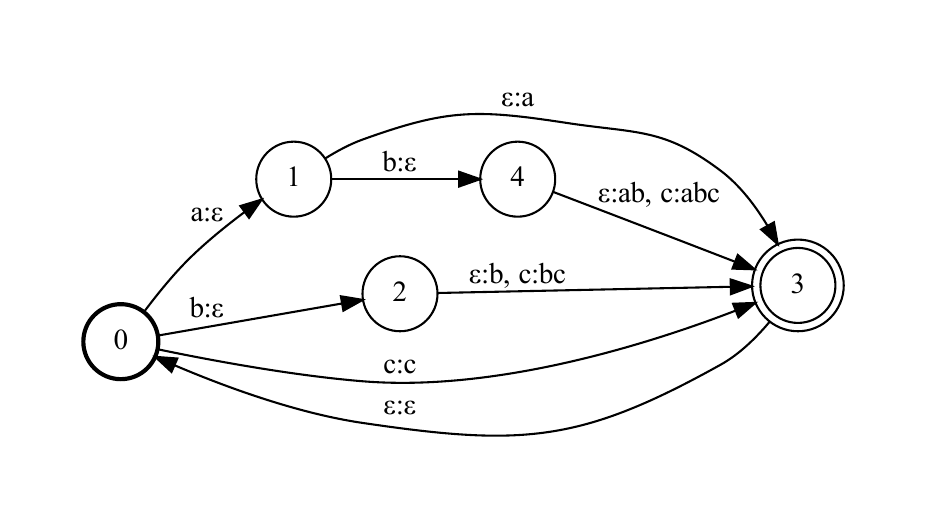}}%
  }
  \hspace{1cm}
\subcaptionbox{A character-to-subword transducer $\mathcal{T}$.\label{sfig:lexicon}}
  {\includegraphics[width=0.42\linewidth]{images/small_lexicon.pdf}}

\subcaptionbox{$\mathcal{A} \circ \mathcal{T}$ \label{sfig:composed}}
  {%
    \includegraphics[width=0.48\linewidth,align=c]{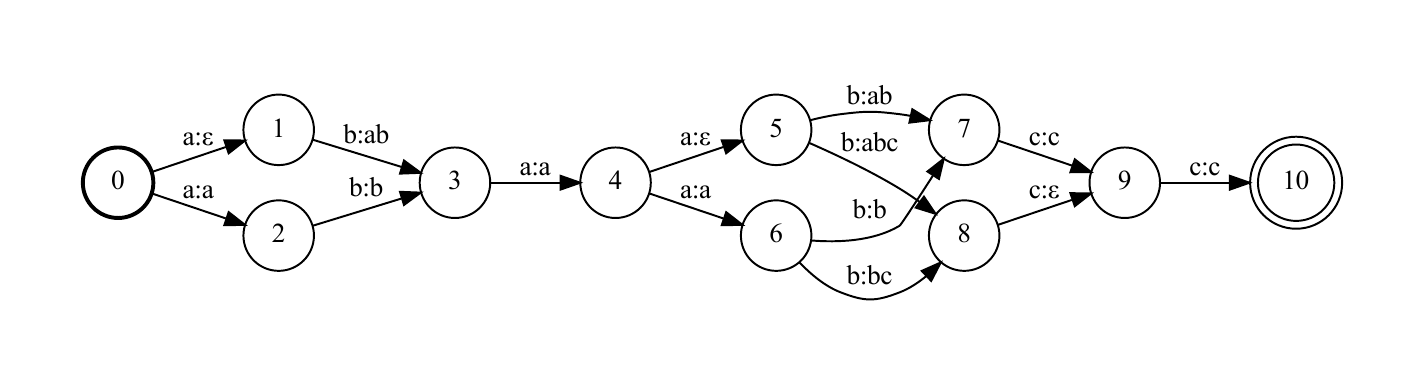}%
    \vphantom{\includegraphics[width=0.48\linewidth,align=c]{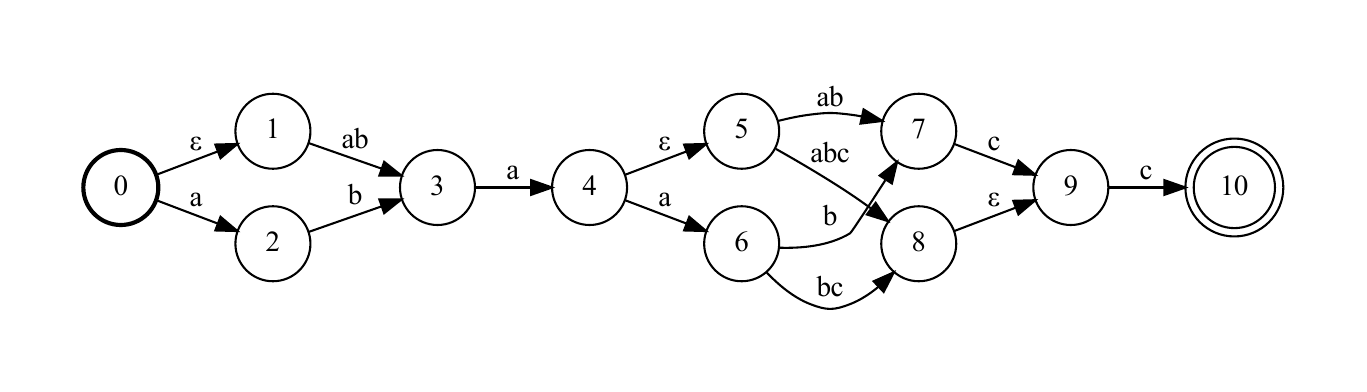}}%
  }
\subcaptionbox{$\textsc{Proj}(\mathcal{A} \circ \mathcal{T})$ \label{sfig:projected}}
  {\includegraphics[width=0.48\linewidth]{images/agnostic_composed_projected.pdf}}

\subcaptionbox{$\textsc{Min}(\textsc{Proj}(\mathcal{A} \circ \mathcal{T}))$ \label{sfig:minimized}}
  {\includegraphics[width=0.6\linewidth]{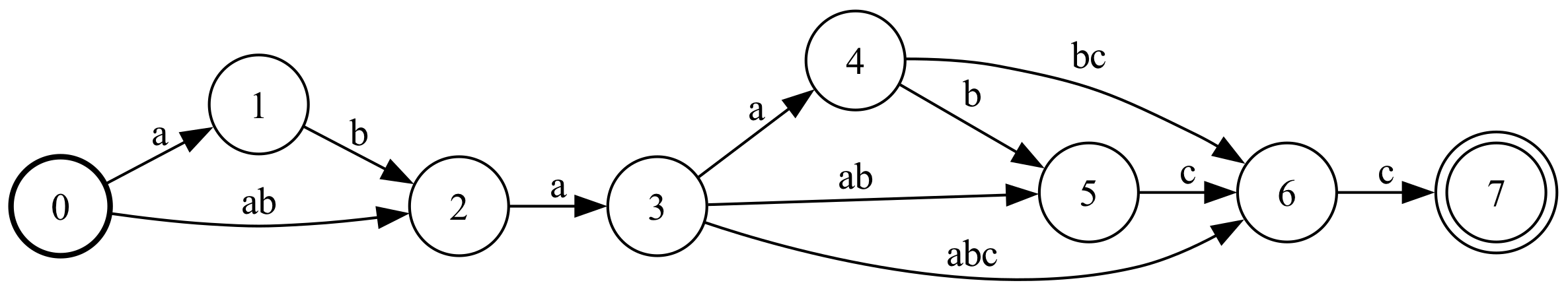}}

\caption{An example of projecting a character-level input pattern $\mathcal{A} = \texttt{abaabcc}$ to the subword level, given a subword vocabulary $\{\texttt{a, b, c, ab, abc, bc}\}$ represented by the character-to-subword transducer $\mathcal{T}$. The intermediate transducers formed during this process are shown in Figures (c) and (d), and the final, minimized subword automaton is given in Figure (e). Observe that for every accepting path in $\textsc{Min}(\textsc{Proj}(\mathcal{A} \circ \mathcal{T}))$, the concatenation of the subwords on that path satisfy the pattern in $\mathcal{A}$ when spelled out character-by-character. For example, $\texttt{ab\textvisiblespace a\textvisiblespace a\textvisiblespace bc\textvisiblespace c}$ and $\texttt{a\textvisiblespace b\textvisiblespace a\textvisiblespace abc\textvisiblespace c}$, which are accepted by $\textsc{Min}(\textsc{Proj}(\mathcal{A} \circ \mathcal{T}))$, both correspond to $\texttt{abaabcc}$, which is accepted by $\mathcal{A}$.
}
\label{fig:compose_steps}

\end{figure*}

\paragraph{Character-Level To Subword-Level Patterns.}
We now have all the pieces to construct a subword-level automaton that matches a given character-level pattern.
Let $\mathcal{T}$ be a lexicon transducer over a subword vocabulary $\Gamma$ (produced by Algorithm \ref{alg:lexicon_construction}) and let $\mathcal{A}$ be a trim, deterministic automaton over $\Sigma^*$ recognizing some given pattern.
The composition $\mathcal{A} \circ \mathcal{T}$ produces a character-sequence-to-subword-sequence transducer, for which for any path $\pi_{x, y}$, $s_i(\pi) = s_o(\pi)$.
Thus, the concatenation of the output labels on any accepting path spells out a character sequence which would have been accepted by $\mathcal{A}$. In fact, this transducer encodes \textit{all} subword sequences with this property.
We then compute:
\begin{equation}\label{eq:min_projection}
    \textsc{Min}(\textsc{Proj}(\mathcal{A} \circ \mathcal{T})),    
\end{equation} which results in a minimal, deterministic automaton recognizing subword sequences which spell out character sequences that satisfy $\mathcal{A}$. Lemma \ref{lma:proj_determinsitic} shows that a determinization step, which can take exponential time in the size of the automaton, is not necessary between the \textsc{Proj} and \textsc{Min} steps.

\begin{restatable}{lemma}{projdet}\label{lma:proj_determinsitic}
Let $\mathcal{A}$ be minimal, deterministic automaton over $\Sigma$ and $\mathcal{T}$ be a character-to-subword transducer over $\Sigma \subseteq \Gamma \subset \Sigma^+$  (Equation \ref{eq:lexicon_formulation}).
Then $\varepsilon\textsc{-Removal}(\textsc{Proj}(\mathcal{A} \circ \mathcal{T}))$ is deterministic.
\end{restatable}

Lemma \ref{lma:proj_determinsitic} implies $\textsc{Min}(\textsc{Proj}(\mathcal{A} \circ \mathcal{T}))$ can be computed in polynomial time in the size of $\mathcal{A}$ and $\mathcal{T}$. First, $\mathcal{A} \circ \mathcal{T}$ can be computed in polynomial time by the standard transducer composition algorithm \cite{DBLP:journals/ijfcs/AllauzenM09}. Likewise, \textsc{Proj} can be done in $O(|\delta|) \in O(|\mathcal{A}||\Gamma|)$ time. Since $\textsc{Proj}(\mathcal{A} \circ \mathcal{T})$ produces an \textit{automaton} over $\Gamma^*$, the standard $\varepsilon$-Removal and $\textsc{Min}$ algorithms can be used, which both have polynomial runtimes in the size of the automaton \cite{DBLP:books/daglib/0086373}.

Figure \ref{fig:compose_steps} gives an example of each of these stages for a given pattern and lexicon transducer.

\section{Tokenization-Preserving Pattern Promotion}\label{sec:tokenization_preserving}

The subword-level pattern promotion described in Section \ref{sec:tokenization_agnostic_pattern_promotion} allows one to form an automaton that encodes all subword sequences which, when concatenated, would match the character-level input pattern. These subword sequences are simply drawn from the subword vocabulary, and do not reflect the underlying tokenization scheme.

For example, given a pattern $\texttt{a\textvisiblespace b\textvisiblespace a\textvisiblespace a\textvisiblespace b}$ and a vocabulary $\Gamma = \{\texttt{a}, \texttt{b}, \texttt{ab}, \texttt{aba}\}$, a tokenization-agnostic subword-level automaton would accept \texttt{a\textvisiblespace b\textvisiblespace a\textvisiblespace a\textvisiblespace b}, \texttt{ab\textvisiblespace a\textvisiblespace ab}, etc. However, a MaxMatch tokenizer would output the tokenization \texttt{aba\textvisiblespace ab} when given this input. This motivates us to determine if we can simultaneously promote character-level patterns to subword-level patterns \textit{and} preserve the underlying tokenization scheme. Formally, rather than the tokenization-agnostic:
\[\{t~\mid~ \textsc{concat}(t) \in L(\mathcal{A})\},\] we want to produce the language: \[\{t~\mid~\textsc{concat}(t) \in L(\mathcal{A}) \wedge T(\textsc{concat}(t)) = t \}, \] where $T$ is a BPE or MaxMatch tokenizer.

On the surface, this seems difficult, as the pattern could describe an infinite set of strings, each with a unique canonical tokenization. However, here we show that both MaxMatch and BPE tokenization can be expressed as finite state transducers. For BPE, this is a particularly surprising result, given the algorithm's resemblance to a context-free grammar (which is strictly more powerful than a finite-state machine), it's use of a stack-like data-structure (a priority queue), and the fact that BPE merges are not performed left to right, while transducers have constant memory and operate in strictly left-to-right order.

\paragraph{$\varphi$-transitions.}
We first introduce a key component used in the constructions for both MaxMatch and BPE transducers. Automata and transducers can be augmented with a special $\varphi$ symbol used to denote \textit{failure} transitions, in the sense of Aho-Corsaick \cite{AhoC75, DBLP:conf/wia/AllauzenR18}. This augmentation does not change the power of automata or transducers (in the computability-theory sense), but instead allows for more succinct and clear representations of some common cases of transducers.

Formally, a $\varphi$-transition from a state $q$ is a transition $(q, \varphi, o, p)$ which acts like an $\varepsilon$ transition in that it does not consume an input symbol, but can only be traversed if the current symbol is not present in any other transition originating from $q$.

\subsection{MaxMatch-Preserving Pattern Promotion}\label{ssec:maxmatch_transducer}

We begin with MaxMatch tokenization, which is rooted in automata theory and is thus somewhat more straightforward to understand intuitively.

\begin{figure*}[t]
\centering
\begin{subfigure}{\linewidth}
    \centering
    \includegraphics[width=0.7\linewidth]{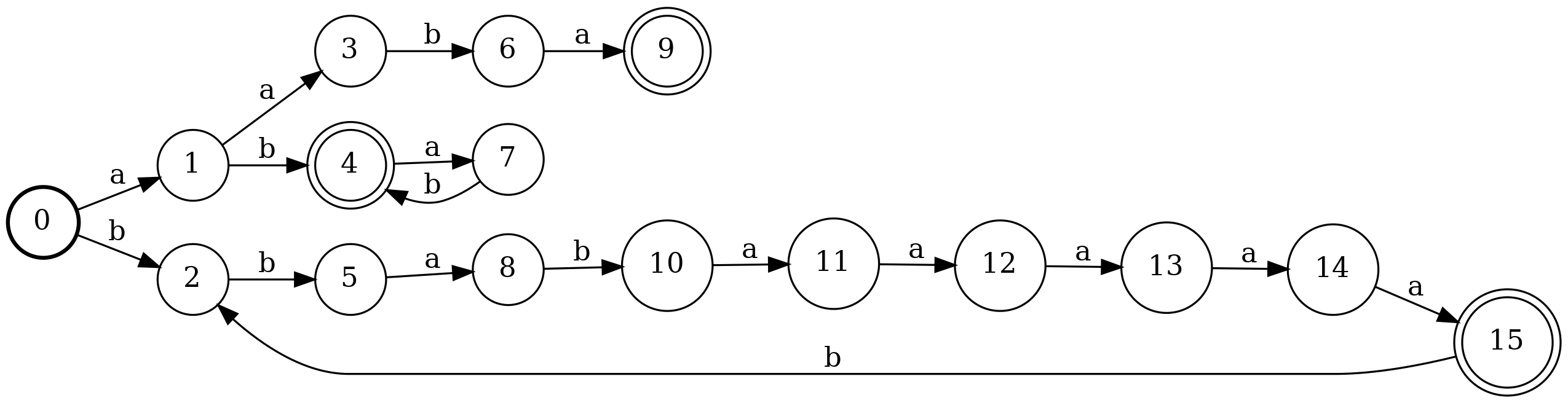} 
    \caption{A character pattern automaton $\mathcal{A}$.}\label{fig:image1}
\end{subfigure}

\begin{subfigure}{.49\linewidth}
    \centering
    \includegraphics[scale=0.5]{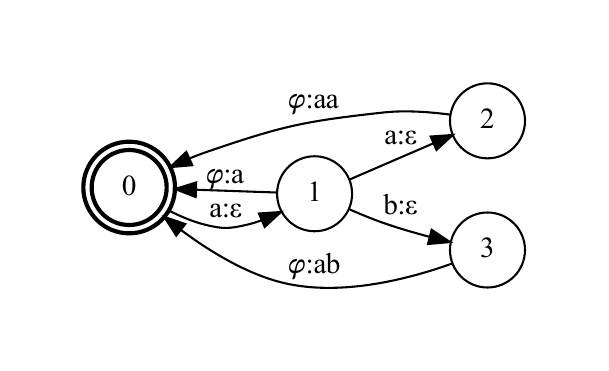} 
    \caption{The MaxMatch-Preserving Transducer $\mathcal{T}_\mathrm{Aho}$}\label{fig:basic_aho_lexicon}
\end{subfigure}
\hfill
\begin{subfigure}{0.49\linewidth}
  \centering
  \includegraphics[scale=0.5]{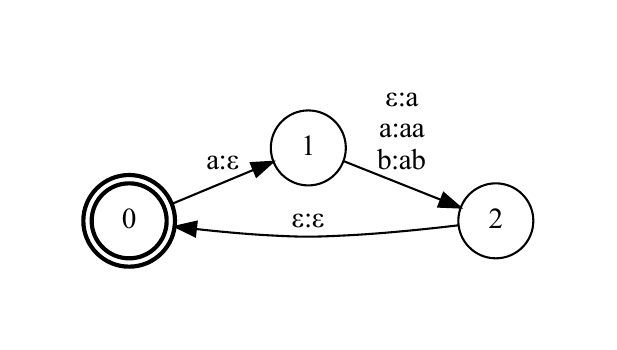} 
  \caption{The Tokenization-Agnostic Transducer $\mathcal{T}$}\label{fig:basic_lexicon}
\end{subfigure} 

\begin{subfigure}{.49\linewidth}
    \centering
    \raisebox{.15\height}{\includegraphics[width=1.0\linewidth]{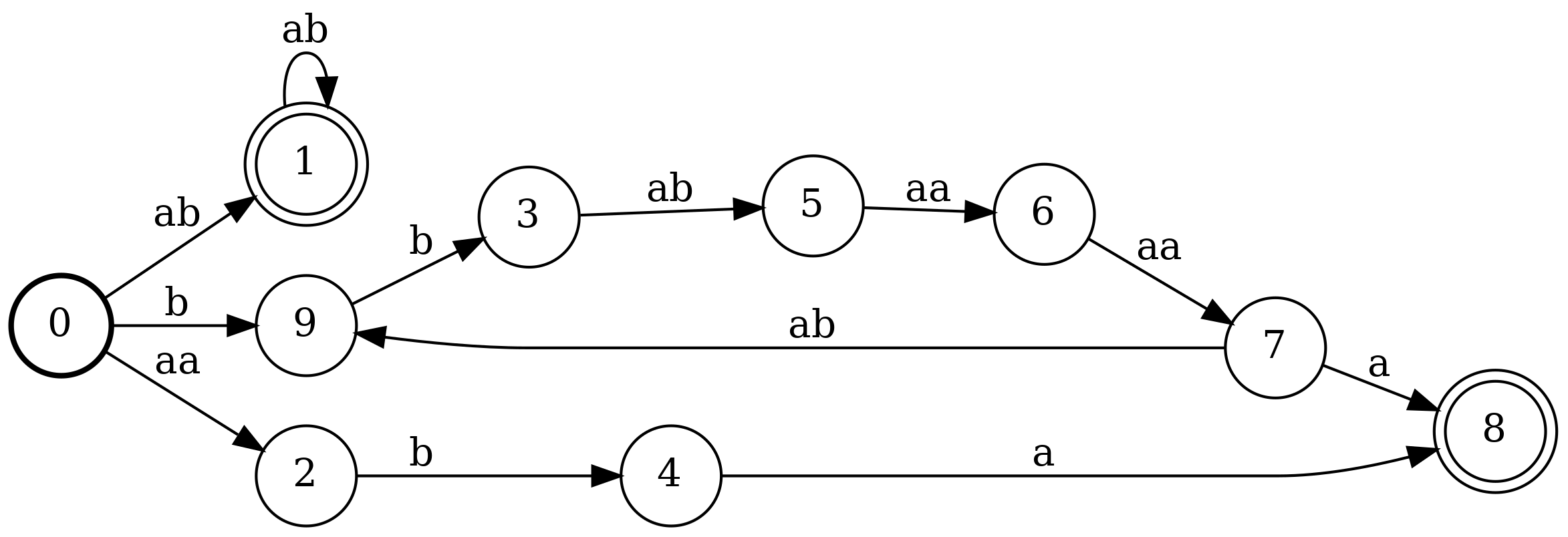} }
    \caption{$\textsc{Min}(\textsc{Proj}(\mathcal{A} \circ \mathcal{T}_{Aho}))$}\label{fig:aho_trie_composed_min}
\end{subfigure}
\hfill
\begin{subfigure}{0.49\linewidth}
  \centering
  \includegraphics[width=1.0\linewidth]{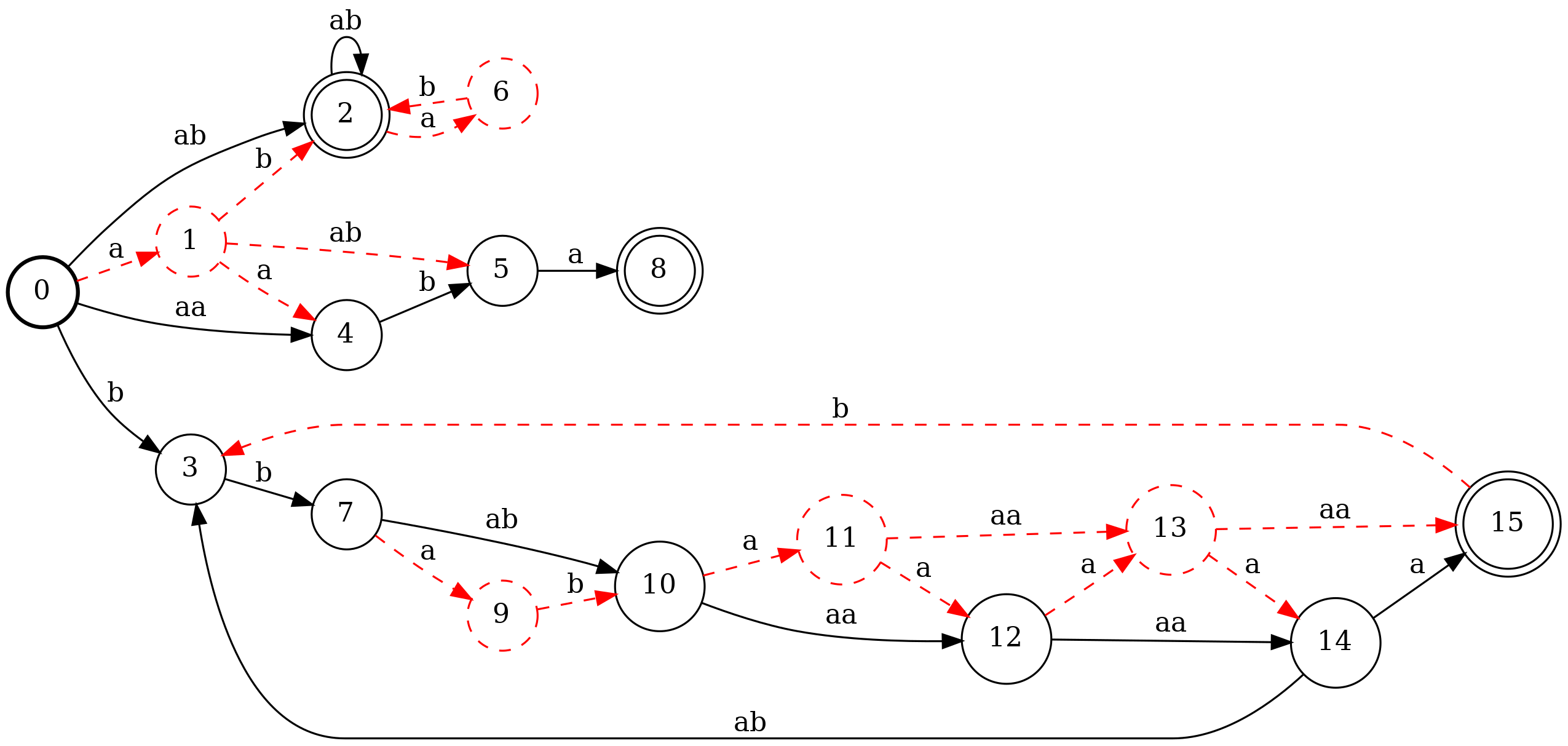} 
  \caption{$\textsc{Min}(\textsc{Proj}(\mathcal{A} \circ \mathcal{T}))$}\label{fig:regular_transducer_composed_min_new_marked_red}
\end{subfigure} 
\caption{A character-level automaton $\mathcal{A}$ is intersected with subword lexicons over \texttt{\{a, b, aa, ab\}} represented by the MaxMatch-preserving transducer $\mathcal{T}_{Aho}$, and the tokenization-agnostic transducer $\mathcal{T}$, shown in Figures (b) and (c), respectively. The results of the composition are shown in (d) and (e). In (e) specifically, arcs and states that appear in the unconstrained automaton but would not appear in the constrained automaton (since they do not encode greedy maximal matches) are shown in dashed-red.}
\label{fig:aho_comparison}
\end{figure*}

We follow \cite{song-etal-2021-fast}, and build a lexicon transducer in the style of the Aho-Corasick automaton \cite{AhoC75}.

Roughly, given a vocabulary $\Gamma$, the standard Aho-Corasick automaton allows one to find all substrings where an input string matches an item from the vocabulary in time linear to the input length. This is done by including a number of \textit{failure arcs} ($\varphi$-transitions) in the automaton---arcs which map a state (that corresponds to a prefix of a token) the state that corresponds to the longest matching \textit{suffix} of that token which is also a prefix of another word in the vocabulary, and which can only be traversed if no suitable character transition exists at the current state. The failure arcs allow for full-token prefixes of the current token to be marked as matched, while moving to a prefix of another token that is currently valid, without recomputation.

\newcite{song-etal-2021-fast} show that this construction can be slightly modified to produce only matches which are greedily as long as possible, which corresponds exactly to the MaxMatch (WordPiece) tokenization inference algorithm. They achieve this by augmenting the failure arcs (referred to as \textit{failure links})
of with a ``popping'' mechanism, which pops the longest prefix token(s) that have already been matched before moving to the suffix state via the failure arc. Like the original Aho-Corasick algorithm, failure arcs and links can be precomputed and stored for fast inference. The authors note that, despite using the algorithmic form of the Aho-Corasick algorithm, their approach can also be viewed as an application of finite-state transducers.

Algorithm \ref{alg:failure_construction} is a simplified version of the failure-trie precomputation algorithm from \cite{song-etal-2021-fast} and Algorithm \ref{alg:aho_construction} shows how to convert the failure-trie into a lexicon transducer, $\mathcal{T}_{Aho}$\footnote{Algorithms \ref{alg:failure_construction} and \ref{alg:aho_construction} are given in Appendix \ref{apx:maxmatch_algorithms} for space reasons, as they are not crucial to the results of this paper.}.

$\mathcal{T}_{Aho}$ is a valid character-to-subword lexicon transducer in that each accepting path corresponds to a sequence of characters from $\Sigma$ on the input and a sequence of subwords from $\Gamma$ on the output such that the concatenation of the subwords is the same as the character sequence. However, differently from the character-to-subword transducers from Section \ref{sec:subword_transducer}, given an input automaton $\mathcal{A}$ to compose with $\mathcal{T}_{Aho}$, there is a one-to-one correspondence between sequences accepted by $\mathcal{A}$ and sequences accepted by $\mathcal{T}_{Aho}$ which corresponds to exactly the greedy parse.

As such, $\mathcal{T}_\mathrm{Aho}$ can be composed with an input pattern $\mathcal{A}$ using the $\varphi$-transition semantics and projected to an automaton over $\Gamma$ via:
$\textsc{Min}(\textsc{Proj}(\mathcal{A} \circ \mathcal{T}_\mathrm{Aho}))$, as in Equation \ref{eq:min_projection}. The result is a subword automaton which only matches subword sequences that would match the character pattern in $\mathcal{A}$ \textit{and} are a greedy tokenization.

Figure \ref{fig:aho_comparison} gives an example of an input automaton being composed with a standard character-to-subword lexicon transducer (as described in Section \ref{sec:subword_transducer}) and an Aho-Corasick character-to-subword lexicon transducer.

\subsection{BPE-Preserving Pattern Promotion} \label{ssec:bpe_transducer}
We now turn to BPE, and attempt to construct a transducer similar to $\mathcal{T}_{Aho}$, but which produces a subword-level automaton that matches the tokenizations induced by a BPE tokenizer $(\Gamma, \mu)$.

Recall Algorithm \ref{alg:bpe_inference}, the canonical algorithm for BPE tokenization. This algorithm processes the input string in arbitrary order (in that it does not necessarily form the final tokenized sequence from the left to the right, and can skip back and forth) and resembles a context-free grammar. This suggests that the relation described by BPE is not regular and cannot be represented by a finite-state transducer.

However, consider the following alternative implementation of BPE \cite{zouhar-etal-2023-formal}:

\begin{algorithm}[H]
\begin{algorithmic}[1]
\For{$(a, b) \in \mu$}
    \State $w \gets \textsc{Apply}((a, b), w)$
\EndFor

\State \Return $w$
\end{algorithmic}
\caption{%
\hspace{-1mm}: Iterative BPE \\
    \textbf{Input}: Word $w \in \Sigma^+$, BPE Tokenizer $\mathcal{B} = (\Gamma, \mu)$ \\
    \textbf{Output}: Tokenized sequence $t \in \Gamma^+$
}
\label{alg:bpe_iterative}
\end{algorithm}

Rather than selecting the highest priority merge available, Algorithm \ref{alg:aho_construction} starts from the character level string and applies \textit{every} merge in order of priority and produces an identical result to Algorithm \ref{alg:bpe_inference}. Intuitively, $(\texttt{ab}, \texttt{c}) \rightarrow \texttt{abc}$ can only be applied if the merge $(\texttt{a}, \texttt{b}) \rightarrow \texttt{ab}$ was already applied.

This alternative algorithm suggests a finite-state transducer representation is possible --- we can simulate BPE by forming, for each $(\texttt{a}, \texttt{b}) \in \mu$, a transducer that maps \texttt{a\textvisiblespace b} $\rightarrow \texttt{ab}$. We call such a transducer a merge \textit{gadget}, $G_{(a, b)}$. The structure of these gadgets is simple, as shown in Figure \ref{fig:merge_gadget}.

\begin{figure}[ht]
\centering
\includegraphics[scale=.6]{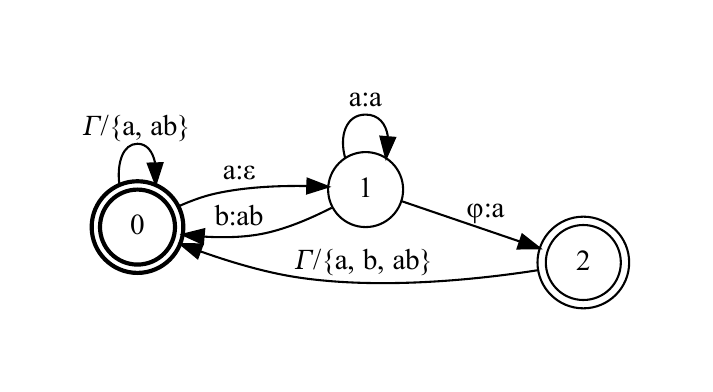}
\caption{A merge gadget $G_{(\texttt{a}, \texttt{b})}$ for the merge $(\texttt{a}, \texttt{b}) \rightarrow \texttt{ab}$. All arcs that don't have an output symbol are assumed to be of the form $(q, c, c, p)$.}
\label{fig:merge_gadget}
\end{figure}

\begin{figure}
\captionsetup[subfigure]{aboveskip=-2pt,belowskip=-2pt}
\begin{subfigure}{\linewidth}
    \centering
    \hspace*{-1cm}\includegraphics[scale=0.6]{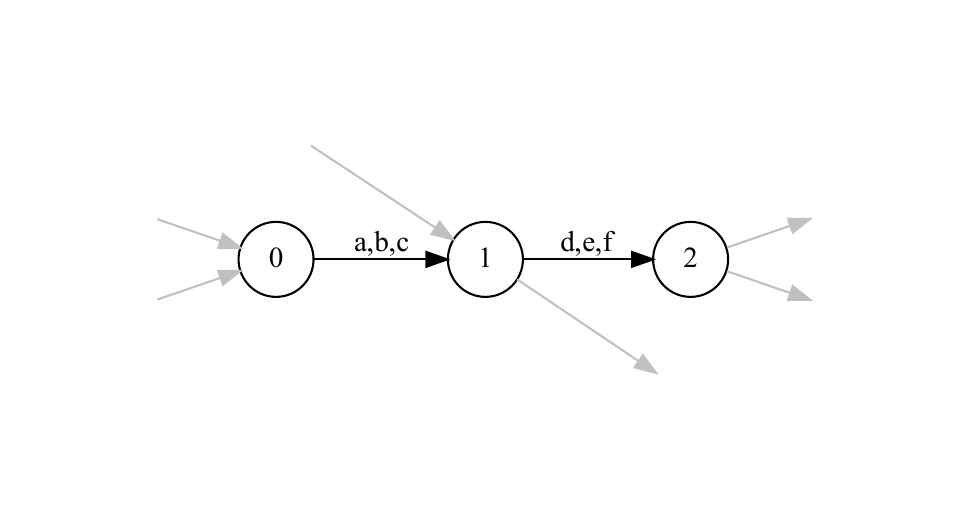} 
    \caption{A subset of an automaton with a path $\texttt{a\textvisiblespace  d}$.}\label{fig:bpe_comp_example_before}
\end{subfigure}
\begin{subfigure}{\linewidth}
    \centering
    \hspace*{-1cm}\includegraphics[scale=0.6]{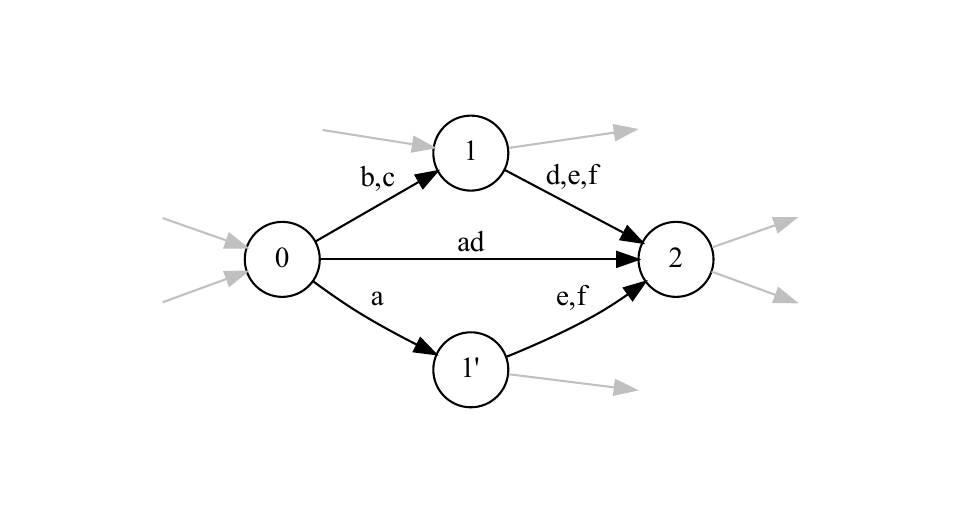} 
    \caption{After composition with $G_{(\texttt{a}, \texttt{d})}$.}\label{fig:bpe_comp_example_after}
\end{subfigure}
\caption{An example of how composing with a merge gadget modifies the input automaton.}
\label{fig:merge_subset_example}
\end{figure}

A gadget contains three states: the \textit{initial} state ($q_0$), the \textit{active} state ($q_1$), and the \textit{abort} state ($q_2$). For a merge gadget $G_{(a, b)}$, the initial state contains a self loop $(q_0, c, c, q_0)$ for each $c \in \Gamma$ such that $c \notin \{a, ab\}$ ($ab$ can't even be observed before this merge is applied). It also contains an arc to the active state $(q_0, a, \varepsilon, q_1)$. The active state contains three arcs, one which resolves the merge $(q_1, b, ab, q_0)$, one which postpones the merge $(q_1, a, a, q_1)$, and one which cleans up the merge if we read a character other than $b$ or have nothing else to consume, $(q_1, \varphi, a, q_2)$. The reason for $\varphi$ in the last arc is that, when reading the initial $a$, we do not emit an $a$ in case we will form a merge. The self-loop on state $q_1$ consumes and produces $a$'s, but if we fail to make a match, we need to produce one additional $a$ to make up for the one that was not generated on the arc from $q_0$ to $q_1$. The abort state contains an arc $(q_2, c, c, q_0), \forall c \ne a, b, ab$, which allows us to go back to the start state after failing to make a match at the active state. 

Figure \ref{fig:merge_subset_example} gives an example of an application of a merge gadget to an automaton. 

To build intuition, suppose we have an automaton $\mathcal{A}$ representing the input string \texttt{b\textvisiblespace c\textvisiblespace a\textvisiblespace b\textvisiblespace a\textvisiblespace b\textvisiblespace c\textvisiblespace c}, a BPE tokenizer $\mathcal{B} = (\Gamma, \mu)$, where $\Gamma = \{\texttt{a}, \texttt{b}, \texttt{c}, \texttt{ab}, \texttt{bc}, \texttt{cc}, \texttt{abc}\}$ and $\mu = [(\texttt{a}, \texttt{b}), (\texttt{b}, \texttt{c)}, (\texttt{c}, \texttt{c}), (\texttt{ab}, \texttt{c})]$, and a list of gadgets $\mathcal{G} = [G_{(\texttt{a}, \texttt{b})}, G_{(\texttt{b}, \texttt{c})},G_{(\texttt{c}, \texttt{c})}, G_{(\texttt{ab}, \texttt{c})}]$.

After applying $\mathcal{A} \circ G_{(a, b)}$, we have a new transducer where the input is the sequence \texttt{b\textvisiblespace c\textvisiblespace a\textvisiblespace b\textvisiblespace a\textvisiblespace b\textvisiblespace c\textvisiblespace c} but the output is the sequence \texttt{b\textvisiblespace c\textvisiblespace ab\textvisiblespace ab\textvisiblespace c\textvisiblespace c}, and all instances of \texttt{a\textvisiblespace b} have been merged to \texttt{ab}. Repeating this for each of the merge gadgets results in \[ \mathcal{A} \circ G_{(a, b)} \circ G_{(b, c)} \circ G_{(c, c)} \circ G_{(ab, c)},\] which transduces \texttt{b\textvisiblespace c\textvisiblespace a\textvisiblespace b\textvisiblespace a\textvisiblespace b\textvisiblespace c\textvisiblespace c} to \texttt{bc\textvisiblespace ab\textvisiblespace ab\textvisiblespace cc}, the exact sequence produced by $\mathcal{B}(\texttt{bcababcc})$. Written out another way, we can implement BPE by computing the composition:
\begin{equation}\label{eq:chained_composition}
    \textsc{Min}\left(\textsc{Proj}\left(\mathcal{A} \circ \left( \mathop{\bigcirc}\limits_{m \in \mu} G_{m} \right) \right)\right).
\end{equation}

Figure \ref{fig:bpe_comparison} gives an example of BPE-preserving pattern promotion given a BPE tokenizer.

\subsubsection{Runtime}
Via gadgets, we have shown that regular languages are closed under BPE transduction. The naïve application of the full merge composition in Equation \ref{eq:chained_composition} takes $O(|\mathcal{A}|\prod_{m \in \mu} |G_{m}|) = O(3^{|\mu|}|\mathcal{A}||\Gamma|)$ time \cite{DBLP:journals/ijfcs/AllauzenM09}. For a BPE tokenizer with just $30$ merges, this would be prohibitively expensive, let alone a tokenizer with $64k$ merges.

However, here we show that the actual runtime is $\textsc{Poly}(|\mathcal{A}|, |\mu|, |\Gamma|)$. In particular, we bound the size of the resulting automaton and show that each consecutive merge can be performed in polynomial time in the size of the input pattern automaton.

\begin{restatable}{theorem}{bpesize}
  Given an automaton $\mathcal{A}$ and a series of merges $m_1, m_2, \dots, m_k \in \mu$. Let $\mathcal{A}' = \textsc{Min}(\textsc{Proj}(\mathcal{A} \circ G_{m_1} \circ G_{m_2} \circ \dots \circ G_{m_{k-1}}))$.
        Then, $\textsc{Min}(\textsc{Proj}(\mathcal{A}' \circ G_{\mu_k}))$ can be computed in $\textsc{Poly}(|\mathcal{A}|, k, |\Gamma|)$ time.
\end{restatable}\label{thm:bpe_complexity}

\begin{restatable}{corollary}{bpesizecorollary}
    The size of 
    \[\textsc{Min}\left(\textsc{Proj}\left(\mathcal{A} \circ \left( \mathop{\bigcirc}\limits_{(a,b) \in \mu} G_{(a,b)} \right) \right)\right)\] is $ \in \textsc{Poly}(|\mathcal{A}|, |\mu|, |\Gamma|).$
\end{restatable}

\begin{figure}
        \centering
        \begin{subfigure}{\linewidth}
            \centering
            \includegraphics[scale=0.5]{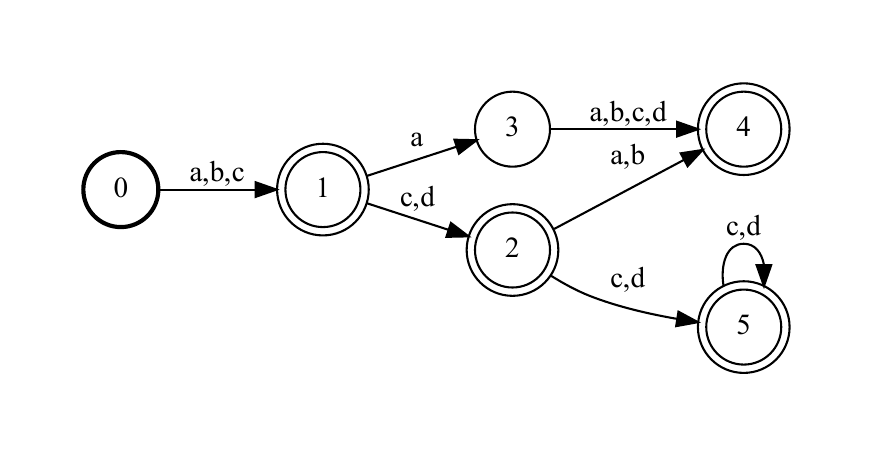}
            \caption{A pattern automaton $\mathcal{A}$.}
        \end{subfigure}
        \begin{subfigure}{\linewidth}
            \centering
            \includegraphics[scale=0.5]{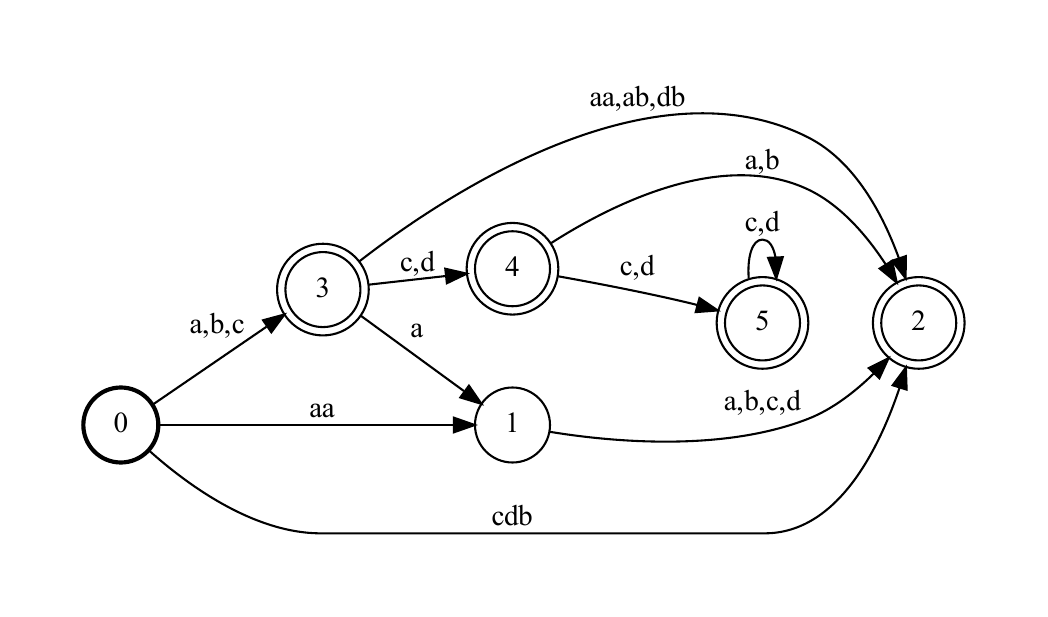}
            \caption{Tokenization Agnostic $\textsc{Min}(\textsc{Proj}(\mathcal{A} \circ \mathcal{T}))$}
        \end{subfigure}
        \begin{subfigure}{\linewidth}
            \centering
            \includegraphics[scale=0.5]{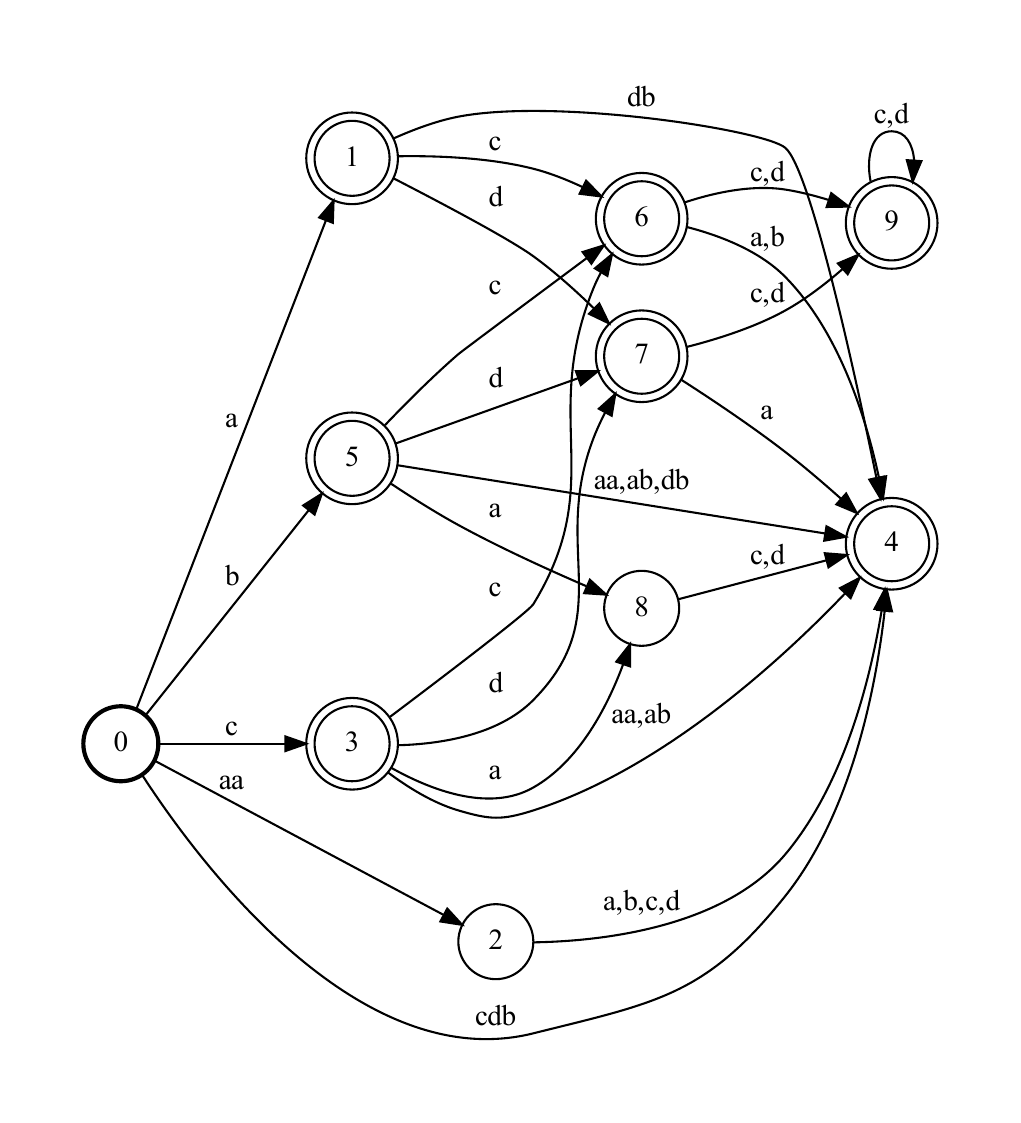}
            \caption{$\textsc{Min}(\textsc{Proj}(\mathcal{A} \circ G_{(a, a)} \circ G_{(a, b)} \circ G_{(d, b)} \circ G_{(c, db)}))$}
        \end{subfigure}

    \caption{An example of BPE-preserving pattern promotion of an automaton $\mathcal{A}$ for a BPE tokenizer $\mathcal{B}$ with $\Gamma = \{\texttt{a}, \texttt{b}, \texttt{c}, \texttt{d}, \texttt{aa}, \texttt{ab}, \texttt{db}, \texttt{cdb}\}$ and $\mu = \langle(\texttt{a}, \texttt{a}), (\texttt{a}, \texttt{b}), (\texttt{d}, \texttt{a}), (\texttt{c}, \texttt{db})\rangle$. Any path in the BPE-preserving pattern automaton (c) matches the pattern from $\mathcal{A}$ and corresponds to the tokenization produced by $\mathcal{B}$, while the tokenization-agnostic promoted pattern (b) accepts non-canonically-tokenized sequences.}
    \label{fig:bpe_comparison}
\end{figure}

\section{Guided Generation}\label{sec:guided_generation}

\textit{Guided generation} is a technique used to constrain the outputs of language models to adhere to a specified pattern \cite{willard2023efficient, koo2024automatabased}. This is necessary since typically neural language models assign non-zero probability to all sequences. If a sequence of a specified pattern is desired, it would be beneficial to be able to sample from that distribution of sequences directly. Guided generation frameworks typically take in a character-level pattern (e.g., a regular expression) and constrain the model to match that pattern by masking out logits in the decoder softmax layer that correspond to tokens which would cause the sequence to fail to match the pattern. Formally, rather than sampling from:
\[\mathcal{P}(w \mid w_1w_2\dots w_k),\]
one would sample from: \[\mathcal{P}(w \mid w_1w_2\dots w_k \wedge w_1w_2\dots w_kw\Sigma^* \in L(\mathcal{A})),\] where $\mathcal{A}$ is some automaton.

Often, $\mathcal{A}$ is more easily specified at the character level, while the language model operates at some subword level. To fix this granularity mismatch, subword-level patterns are used \cite{willard2023efficient, koo2024automatabased}. However, prior work constrains models only to match the original pattern $\mathcal{A}$ without respect to the underlying tokenization scheme. Here we discuss why that might cause issues with modeling.

First, we note that language models are not conditioned on the \textit{surface form} of the text, but rather the exact tokenization of the text. Thus, the same text tokenized in two different ways will map to different representations within the model.

Imagine we wish to constrain a language model trained with a MaxMatch-style tokenizer to output just the text $\texttt{racecar}$ (encoded by $\mathcal{A}$). At the start of decoding, it is drawing from the distribution:

\[ \mathcal{P}(w \mid \texttt{<sos>},w\Sigma^* \subseteq \mathcal{A}).\]

Because there is no context and the model is not explicitly aware that it's outputs are being constrained, the model may choose a higher probability token that still that prefix-matches the regular expression --- for example, \texttt{r}. Then, in the next inference step, the next token is drawn from: 
\[ \mathcal{P}(w \mid \texttt{<sos>\textvisiblespace r},~\texttt{r}w\Sigma^* \subseteq \mathcal{A}).\]
After several decoding steps, imagine that the model has generated the token sequence \texttt{r\textvisiblespace a\textvisiblespace ce\textvisiblespace car}, which does indeed satisfy the prescribed regular expression. However, the canonical tokenization on which the model was trained is \texttt{race\textvisiblespace car} (i.e., this is the token sequence that the MaxMatch tokenizer would have produced given the character sequence \texttt{r\textvisiblespace a\textvisiblespace c\textvisiblespace e\textvisiblespace c\textvisiblespace a\textvisiblespace r}). The model was only able to generate a sequence with surface form that matched the input pattern due to the increasingly strict constraints on its output at each step. As typical neural language models are conditioned not on the surface form of the text, but rather on the explicit choice of tokenization, this can have large downstream effects. That is, the distribution:
\[\mathcal{P}(w \mid \texttt{<sos>\textvisiblespace r\textvisiblespace a\textvisiblespace ce\textvisiblespace car}) \]
can be much different than the distribution:
\[\mathcal{P}(w \mid \texttt{<sos>\textvisiblespace race\textvisiblespace car}), \]
despite the context having the same surface form.

One way to counteract this is to periodically stop decoding, detokenize and retokenize the sequence, and then resume decoding. This will put the text back in to the canonical tokenized form that the model was trained on. However, this is a slow and expensive step (since the model must re-embed the entire context) and does not fully resolve the issue since the current text was still generated using non-canonical tokenizations.

Thus, rather than using tokenization-agnostic pattern promotion, the practitioner can use the tokenization-aware pattern-promotion constructions described in Sections \ref{ssec:maxmatch_transducer} and \ref{ssec:bpe_transducer} to avoid the issues introduced in this section by not only forcing the model to ahere to some surface-text pattern, but also to the tokenization scheme that it was trained on.

\section{Conclusion}

We formalized character-to-subword-level regular expression pattern promotion through a simple finite-state transduction framework. This framework was then extended to enforce dual, tokenizer-aware constraints: the promoted subword pattern simultaneously accepts only sequences which match the original character-level pattern, but do so while also only allowing sequences that would match an underlying BPE or MaxMatch tokenizer's output. While the MaxMatch-preserving construction is rooted in classical automata theory via the Aho-Corasick algorithm, the BPE construction is a novel and surprising result, given that BPE is, on the surface, seemingly incapable of being captured by finite-state transducers. Further, we show that, contrary to the standard complexity analysis of the BPE construction, BPE-preserving pattern promotion can be done in polynomial time in the size of the input automaton and BPE tokenizer.

The finite-state transduction framework and its BPE and MaxMatch extensions have an important application to guided generation. While subword-level patterns can constrain the outputs of LLMs to match a specific pattern, our tokenization-aware pattern promotions can do this while also adhering to the tokenization scheme that the model was trained on (and thus has an inductive bias towards).

\bibliography{references.bib}
\bibliographystyle{acl_natbib}

\onecolumn

\appendix

\section{Tokenization-Agnostic Lexicon Construction}\label{apx:tokenization_agnostic_lexicon}

\begin{algorithm}[ht]
{\fontsize{9}{9}\selectfont
\begin{algorithmic}[1]
\State $q_\mathrm{start} \gets \text{new state}$, $Q \gets \{q_\mathrm{start}\}$, $F \gets \{\}$
\State $\delta \gets \text{empty transition function}$
\For{$w \in \Gamma$}
    \State $\mathrm{cur} \gets q_\mathrm{start}$
    \For{$c \in w$}
        \State $q' \gets \text{new state}$
        \State $Q \gets Q \cup \{q'\}$
        \If{$c \text{ is last character of } w$}
            \State $\delta \gets \delta \cup \{(\mathrm{cur}, c, w, q')\}$ \Comment{Arc emitting $w$}
        \Else
                \State $\delta \gets \delta \cup \{(\mathrm{cur}, c, \varepsilon, q')\}$
        \EndIf
        \State $\mathrm{cur} \gets q'$
    \EndFor
    \State $q' \gets \text{new state}$
    \State $Q \gets Q \cup \{q'\}$
    \State $\delta \gets \delta \cup \{(\mathrm{cur}, \varepsilon, w, q')\}$
    \State $F \gets F \cup \{q'\}$ \Comment{$q'$ is end of word}
\EndFor
\State $\mathcal{T} \gets \varepsilon\textsc{-Close}(\textsc{Determinize}((\Sigma, \Gamma, Q, q_\mathrm{start}, F, \delta)))$ \label{line:minimize}
\State \Return $\mathcal{T}$ 
\end{algorithmic}
}
\caption{%
    \hspace{-1mm}: Lexicon Construction\newline
    \textbf{Inputs}: Subword Vocabulary $\Sigma \subseteq \Gamma \subset \Sigma^+$
}
\label{alg:lexicon_construction}
\end{algorithm}

\section{MaxMatch Transducer Construction}\label{apx:maxmatch_algorithms}
We briefly reintroduce the terminology from \cite{song-etal-2021-fast}.
Let $\Gamma$ be a subword vocabulary, and let $T$ be a trie with state space $V$ such that each state in the trie corresponds to a prefix of a token in $\Gamma$. Let $v \in V$ be a state in the trie corresponding to prefix $w_1w_2\dots w_k$. Define two functions $f(v) : V \rightarrow V \cup \{\texttt{null} \}$, the failure link function, and $F(v) : V \rightarrow \Gamma^*$, the failure pop function. $f(v)$ computes the state $v$ which corresponds to the suffix $w_kw_{k+1}\dots w_n$ that remains after chunking $w_1w_2\dots w_{k-1}$ into tokens $t_1, t_2, \dots t_l$ where $t_i \in \Gamma$ and each corresponds to the longest possible matching token at each step. Then, $F(v) = [t_1, t_2, \dots t_k]$. Algorithm \ref{alg:failure_construction} does not differentiate between marked and unmarked subwords (e.g., \texttt{token} vs \texttt{\#\#token}), but can easily be extended to handle this case.

\begin{minipage}{0.47\textwidth}
\begin{algorithm}[H]
{\centering \fontsize{9}{9}\selectfont
\begin{algorithmic}[1]
\State $\mathrm{root} \gets \Call{Node}{\null}$ \Comment{Build a trie representing $\Gamma$}
\For{$w \in \Gamma$}
    \State $\mathrm{cur} \gets \mathrm{root}$
    \For{$c \in w$}
        \If{$c \notin \mathrm{cur}.\mathrm{children}$}
            \State $\mathrm{cur}.\mathrm{children}[c] \gets \Call{Node}{\null}$
            \State $\mathrm{cur}.\mathrm{children}[c].\mathrm{str} \gets \mathrm{cur}.\mathrm{str} + c$
        \EndIf
        \State $\mathrm{cur} \gets \mathrm{cur}.\mathrm{children}[c]$
    \EndFor
    \State $\mathrm{cur}.\mathrm{final} = \texttt{True}$ \Comment{Corresponds to $w \in \Gamma$}
\EndFor

\State $\mathrm{queue} \gets [\mathrm{root}]$ \Comment{Begin computing failure arcs}

\While{$\mathrm{queue}$ is not empty}
    \State $\mathrm{cur} \gets \mathrm{queue}.\textsc{Dequeue}()$
    \For{$c \in \mathrm{cur}.\mathrm{children}$}
        \State $v \gets \mathrm{cur}.\mathrm{children}[c]$
        \If{$v.final$}
            \State $f(v) \gets r$
            \State $F(v) \gets [\mathrm{cur}.\mathrm{str}]$
        \Else
            \State $z \gets f(\mathrm{cur})$
            \State $Z \gets []$
            \While{$z \ne \texttt{null}$}
                \State $\textsc{Extend}(Z, F(z))$
                \State $z \gets f(z)$
            \EndWhile
            \If{$z \ne \texttt{null}$}
                \State $f(v) \gets z.\mathrm{children}[c]$
                \State $F(v) \gets \textsc{Extend}(F(\mathrm{cur}), Z)$
            \EndIf
        \EndIf
    \State $\mathrm{queue}.\textsc{Enqueue}(v)$
    \EndFor
\EndWhile
\State \Return $\mathrm{root}, f, F$
\end{algorithmic}
}
\caption{%
\hspace{-1mm}: Failure Function Computation \cite{song-etal-2021-fast} \\
    \textbf{Input}: Vocabulary $\Gamma$ \\
    \textbf{Outputs}: Trie $\mathrm{root}$, Failure Functions $f/F$
}
\label{alg:failure_construction}
\end{algorithm}
\end{minipage}
\hfill
\begin{minipage}{0.47\textwidth}
\begin{algorithm}[H]
{\centering \fontsize{9}{9}\selectfont
\begin{algorithmic}[1]
\State $Q \gets \{q_\mathrm{root}\}$ \Comment{State set}
\State $\mathrm{queue} \gets [\mathrm{root}]$
\While{$\mathrm{queue}$ is not empty}
    \State $cur \gets \mathrm{queue}.\textsc{Dequeue}()$
    \For{$c \in \mathrm{cur}.\mathrm{children}$}
        \State $v \gets \mathrm{cur}.\mathrm{children}[c]$
        \State $\textsc{Append}(Q, q_v)$
        \State $\delta \gets \delta \cup (q_\mathrm{cur}, c, \varepsilon, q_v)$
        \State $\mathrm{queue}.\textsc{Enqueue}(v)$
    \EndFor
    \If{$f(\mathrm{cur}) \ne \texttt{null}$} \Comment{Add $\varphi$-arcs}
        \State $q_\mathrm{tmp} \gets q_\mathrm{cur}$
        \For{$w \in F(\mathrm{cur})$}
            \If{$w$ is the last element in $F(\mathrm{cur})$}
                \State $\delta \gets \delta \cup (q_\mathrm{tmp}, \varphi, w, q_{f(\mathrm{cur})})$
            \Else
                \State $q_\mathrm{next} \gets \text{ new state}$
                \State $\textsc{Append}(Q, q_\mathrm{next})$
                \State $\delta \gets \delta \cup (q_\mathrm{tmp}, \varphi, w, q_\mathrm{next})$
                \State $q_\mathrm{tmp} \gets q_\mathrm{next}$
            \EndIf
        \EndFor
    \EndIf
\EndWhile
\State \Return $(Q, \Sigma, \Gamma, \delta, q_\mathrm{root}, \{q_\mathrm{root}\})$ \Comment{$\mathcal{T}_{Aho}$}
\end{algorithmic}
}
\caption{%
    \hspace{-1mm}: Trie-to-Transducer Conversion \\
    \textbf{Inputs}: Trie $\mathrm{root}$, Failure Functions $f/F$, Vocabulary $\Gamma$ over $\Sigma$ \\
    \textbf{Output}: Aho-Corasick Lexicon transducer $\mathcal{T}_{Aho}$
}
\label{alg:aho_construction}
\end{algorithm}
\end{minipage}

\section{Proofs}

\projdet*

\begin{proof}
Let $\mathcal{T}$ be a transducer which recognizes Equation \ref{eq:lexicon_formulation} and was built by Algorithm \ref{alg:lexicon_construction}. By construction, each path from $q_0$ to some $q \in F$ that does not contain an $(\varepsilon, \varepsilon)$ transition (that is, it does not return to the start state) contains only a single non-$\varepsilon$ output label.

Thus, during composition, suppose there is some created state $(q, q_{start})$, where $q \in \mathcal{A}$ and $q_{start}$ is the start state of the lexicon transducer. Let $(p, q_{start})$ be a state that is reached after reading some path with exactly one non-$\varepsilon$ output label. Then, after projection, this path is a sequence of $\varepsilon$ transitions followed by a single non-$\varepsilon$ transition. This can be contracted to a single transition from from $(q, q_{start})$ to $(p, q_{start})$ labeled with the non-$\varepsilon$ label from $\Gamma$. Since each path from the start state to a final state in $\mathcal{T}$ contains a single, unique output label, there cannot be more than one path leading out of $(q, q_{start})$ with that label, thus that state has a deterministic transition function.
\end{proof}

\begin{proposition}\label{prop:multiplicity}
Consider an automaton $\mathcal{A}$ and a single merge $G_{(a, b)}$, and a triplet of states $x, y, z \in Q$ such that $\delta(x, a) = y$, $\delta(y, b) = z$. During the composition of $\mathcal{A} \circ G_{(a, b)}$ over $x, y, z$ at most one new distinguishable state is created.
\end{proposition}

\begin{proof}
    First, we note that if there is no such triplet $x, y, z \in Q$ where $\delta(x, a) = y$, $\delta(y, b) = z$, then the composition $\mathcal{A} \circ G_{(a, b)}$ will describe exactly the same language as $\mathcal{A}$ by definition, so their minimal forms are identical and no new state is created. So, we assume that such a triplet exists. The composition of an automaton with a merge gadget must eliminate the possibility of reading $\texttt{a\textvisiblespace b}$.

    Thus, the state $y$ must be split to eliminate this sequence. Let $y'$ and $y''$ be two new states. Then, for each $c \in \delta(x)$ s.t. $c \ne a$, set $\delta(x, c) = y'$. And, for all transitions $\delta(q, c) = y$ set $\delta(q, c) = y'$. Finally, set $\delta(q, a) = y''$. Then, set $\delta(y') = \delta(y)$ and $\delta(y'', c) = \delta(y, c), \forall c \ne b$.

    By construction, all paths that passed through $y$ are preserved, except for the path from $x \rightarrow y \rightarrow z$ reading $\texttt{a\textvisiblespace b}$ and the original state $y$ can be removed (alternatively, $y'$ can be constructed from $y$ in place).
\end{proof}

\begin{remark}
Figure \ref{fig:merge_subset_example} gives an example of Proposition \ref{prop:multiplicity}.
\end{remark}

\begin{remark}
 We call $y'$ and $y''$ \textit{derived} states from the triplet $x, y, z$ as they are the states that are derived from $y$ after a merge gadget is applied and splits the state. We note that a derived state is specific to the exact triplet (i.e., the derived states of $y$ from the triplet $x, y, z$ are different from the derived states of $y$ from the the triplet $q, y, z$, where $q \ne x$).
\end{remark}

\begin{lemma}\label{lma:k_distuingishable}
Let $\mathcal{A}$ be a trim, deterministic automaton and $m_1, m_2, \dots, m_k$ be a series of merges. Assume there is a triplet of states $x, y, z \in \mathcal{A}$ and let $Y = \{y, y^1, y^2, \dots, y^l\}$ be the set of derived states of $y$ some number of times during the compositions with $\{G_{m_i}\}_{i}^k$. Then, $Y$ contains at most $k+1$ distinguishable states.
\end{lemma}
\begin{proof}
This follows directly from the determinism of $\mathcal{A}$. Since $\mathcal{A}$ is deterministic, for any merge gadget, there is at most one suitable triplet from $\{x\}\times Y \times \{z\}$ which the merge gadget can be applied to (as otherwise either $x$ or some $y' \in Y$ would contain two arcs with the same label, violating determinism). Thus, for each merge, the singular applicable triplet can be split at most once per merge according to Proposition \ref{prop:multiplicity}. Since there are at most $k$ merges, then a derived state can be split at most $k$ times, so $Y$ has a maximum cardinality of $k+1$.
\end{proof}

    So far we have only considered how successive merges act on a set triplet $x, y, z$ or a set of triplets $\{x\} \times Y \times \{z\}$ where $Y$ is the set of derived states from the triplet $x, y, z$. Now we are concerned the general case $X \times Y \times Z$. We consider this by case analysis.

\begin{lemma}\label{lma:general_case}
Let $\mathcal{A}$ be a trim, deterministic automaton (with states $Q$), $x,y, z \in Q$ be a triplet and $X, Y, Z$ be some set of derived states where $x \in X$, $y \in Y$, and $z \in Z$. Let $x', x'' \in X$, $y', y'' \in Y$, and $z', z'' \in Z$ be arbitrary states in the derived state sets. Consider a merge gadget $G_{(a, b)}$ and let $x', y', z'$ and $x'', y'', z'' \in X, Y, Z$ be two triplets of states. Then the application of $G_{(a, b)}$ to $x', y', z'$ and $x'', y'', z''$ produces at most one new indistinguishable state.
\end{lemma}

\begin{proof}
We proceed by case analysis and show that for each choice of $x^1, y^1, z^1$ and $x^2, y^2, z^2$ then $G_m$ is either not applicable over the triplet or produces exactly one new indistinguishable state. There are 8 possible cases, depending on if $x^1 \stackrel{?}{=} x^2$, etc. We assume that, if the case is possible, then the appropriate transitions for the merge $(a, b)$ exist, as otherwise it is guaranteed that no new states are added.

    \begin{enumerate}
        \item[\textbf{Case 1): }] $x^1 = x^2, y^1 = y^2, z^1 = z^2$ \\
            This case is covered by Proposition \ref{prop:multiplicity}.
        \item[\textbf{Case 2): }] $x^1 = x^2, y^1 = y^2, z^1 \ne z^2$ \\
            This is not possible, due to determinism ($y^1$ must have two arcs labeled $b$ but which go to different states, which violates the determinism of $\mathcal{A}$).
        \item[\textbf{Case 3): }] $x^1 = x^2, y^1 \ne y^2, z^1 = z^2$ \\
            This is not possible, due to determinism ($x^1$).
        \item[\textbf{Case 4): }] $x^1 = x^2, y^1 \ne y^2, z^1 \ne z^2$ \\
            This is not possible, due to determinism ($x^1$).
        \item[\textbf{Case 5): }] $x^1 \ne x^2, y^1 = y^2, z^1 = z^2$ \\
            Assume that the gadget operating on $x^1, y^1, z^1$ creates two splits $y^1$ into $q'$ and $q''$ while $x^2, y^1, z^1$ splits $y^1$ into $p', p''$. Like Proposition \ref{prop:multiplicity}, the transition functions are remapped like:        
        \begin{itemize}
            \item $\delta(x^1, c) = p', \forall c \in \{c \mid c \in \delta(x^1) \wedge c \ne a \wedge \delta(x^1, c) = y^1$\}
            \item $\delta(x^1, a) = p''$
            \item $\delta(p') = \delta(y^1)$
            \item $\delta(p'', c) = \delta(y^1, c)$, $\forall c \ne b \in \delta(y)^1$

            \item $\delta(x^2, c) = q', \forall c \in \{c \mid c \in \delta(x^2) \wedge c \ne a \wedge \delta(x^2, c) = y^1$\}
            \item $\delta(x^2, a) = q''$
            \item $\delta(q') = \delta(y^1)$
            \item $\delta(q'', c) = \delta(y^1, c)$, $\forall c \ne b \in \delta(y)^1$
        \end{itemize}
        
        Then, by construction $q'$ and $p'$ are indistinguishable and can be merged, and likewise for  $q''$ and $p''$. So since state $y^1$ is removed and replaced with $q'$ and $q''$, only one new state is added.

        \item[\textbf{Case 6): }] $x^1 \ne x^2, y^1 = y^2, z^1 \ne z^2$ \\
            This case is not possible, due to determinism ($y^1$).
        \item[\textbf{Case 7): }] $x^1 \ne x^2, y^1 \ne y^2, z^1 = z^2$ \\
            This case has two subcases. First, consider some merge $G_{(a', b')}$ which operated on some set of states $q, x, p$ to create $x^1$ and $x^2$. Then, $\delta(x^1, b') = \delta(x^2, b')$ for all $b' \ne d$. So if $b \ne d$, this case is impossible.

            Next, assume $b = d$, then WLOG, there cannot be an arc $\delta(x^1, b)$, as it was removed to prevent reading $\texttt{a}'\textvisiblespace \texttt{b}'$ and so this triplet is not compatible with the merge and no new states are added.

            Thus, either this case is not possible.

        \item[\textbf{Case 8): }] $x^1 \ne x^2, y^1 \ne y^2, z^1 \ne z^2$ \\
            This case is not possible, due to a combination of the reasoning of Case 7 and determinism. In short, $y^1 = y^2$ or else this is case is impossible, and if $y^1 = y^2$ then $z^1 \ne z^2$ implies $\mathcal{A}$ is not deterministic.

    \end{enumerate}

    For all 8 cases, we have shown that either a merge is not able to operate on the specified pairs of triplets, or that, regardless of the choice of triplet, only one new indistinguishable state is added.

\end{proof}

\begin{proposition}\label{prop:equivalent_compose}
\[\textsc{Min}(\textsc{Proj}(\mathcal{T}_1 \circ \mathcal{T}_2 \circ \dots \circ \mathcal{T}_n)) = \textsc{Min}(\textsc{Proj}(\dots \textsc{Min}(\textsc{Proj}(\textsc{Min}(\textsc{Proj}(\mathcal{T}_1 \circ \mathcal{T}_2)) \circ \mathcal{T}_3)) \dots \circ \mathcal{T}_n))\]
\end{proposition}

\begin{proof}
    This follows directly from the definition of composition and projection. Minimization is not relevant as it does not change the language of the transducer.
\end{proof}

\begin{proposition}\label{prop:bpe_min}
Let $\mathcal{A}$ be a trim, deterministic automaton and $G_{(a, b)}$ be a BPE merge. Then $\varepsilon\textsc{-Removal}(\textsc{Proj}(\mathcal{A} \circ G_{(a, b)}))$ is deterministic.    
\end{proposition}

\begin{proof}
Only the arc from state $0$ to $1$ has an $\varepsilon$-output. States $0$ and $2$ are output-deterministic, but state $1$ has an $a$ output transition on two arcs. However, by $\varphi$-semantics, only one of those arcs can be taken from any state by definition. Thus, at each state, there is only one valid transition per output symbol. Since $\mathcal{A}$ is also deterministic, each state has a deterministic transition function, up to $\varepsilon$-closure. There are no $\varepsilon$-cycles, and the only $\varepsilon$-transition is between state $0$ and $1$. Since state $0$ does not have $a$ as an output label (but state $1$ does) the $\varepsilon$-closure does not violate determinism.
\end{proof}

\bpesize*

\begin{proof}
Consider any triplet $x, y, z \in Q$, of which there are $|Q|^3$. By Proposition \ref{prop:multiplicity}, Lemma \ref{lma:k_distuingishable}, and Lemma \ref{lma:general_case}, each triplet can only form $k$ new states as a result of composition with $\bigcirc_i^k G_{m_i}$. As a result, the size of $|\mathcal{A}'| \in O(k|\mathcal{A}|^3)$. Proposition \ref{prop:bpe_min} shows that any $\textsc{Min}(\textsc{Proj}(A \circ G_m))$ can be done in polynomial time. Thus, via induction, following Proposition \ref{prop:equivalent_compose}, $\textsc{Min}(\textsc{Proj}(\mathcal{A}' \circ G_{\mu_k}))$ can be computed in $O(k|\mathcal{A}|^3)$ time by iteratively computing $\textsc{Min}(\textsc{Proj}(\dots\textsc{Min}(\textsc{Proj}(\textsc{Min}(\textsc{Proj}(\mathcal{A} \circ G_{m_1})) \circ G_{m_2})) \dots \circ G_{m_k}))$. The maximum size of any intermediate automaton is $O(k|\mathcal{A}^3|)$ and so composition, projection, and minimization can be done in polynomial time.
\end{proof}

\bpesizecorollary*

\begin{proof}
Let $k = |\mu|$. Note $|\mu| \in O(|\Gamma|)$ by definition. Theorem \ref{thm:bpe_complexity} shows that \[\textsc{Min}\left(\textsc{Proj}\left(\mathcal{A} \circ \left( \mathop{\bigcirc}\limits_{(a,b) \in \mu} G_{(a,b)} \right) \right)\right)\] can be computed in $\textsc{Poly}(|\mathcal{A}|, |\mu|, |\Gamma|)$ time. This implies the size of the output must also be polynomial in $|\mathcal{A}|$ and $|\Gamma|$.
\end{proof}

\end{document}